%% file: main.tex
\documentclass{article} % For LaTeX2e
\usepackage{iclr2027_conference,times}
\usepackage{hyperref}
\usepackage{url}
\usepackage{booktabs}
\usepackage{adjustbox}     % max width=\textwidth to keep wide tables within the page
\usepackage{microtype}
\usepackage{xcolor}
\usepackage{subcaption}
\usepackage[most]{tcolorbox}
\usepackage{wrapfig}
\usepackage[T1]{fontenc}
\input{subtex/package}

\input{subtex/math}

\input{subtex/macro}

\usepackage{multirow}
\usepackage{colortbl}      % for \cellcolor

\definecolor{RoyalBlue}{rgb}{0.25, 0.41, 0.88}

\colorlet{myred}{red!85!black}

\usepackage{enumitem} % for itemize leftmargin
\usepackage[capitalise]{cleveref}

\crefname{equation}{}{}
\crefname{figure}{Fig.}{Figs.}
\crefname{section}{Sec.}{Secs.}
\crefname{appendix}{App.}{Apps.}
\crefname{table}{Tab.}{Tabs.}
\crefname{theorem}{Thm.}{Thms.}
\crefname{lemma}{Lem.}{Lems.}
\crefname{assumption}{Assump.}{Assumps.}
\crefname{proposition}{Prop.}{Props.}
\crefname{corollary}{Cor.}{Cors.}
\crefname{definition}{Def.}{Defs.}
\crefname{remark}{Rmk.}{Rmks.}
\theoremstyle{remark}
\crefname{algocf}{Alg.}{Algs.}
\crefname{algorithm}{Alg.}{Algs.}

\usepackage{color,soul}

\usepackage{multirow}
\usepackage{colortbl}      % for \cellcolor
\definecolor{RoyalBlue}{rgb}{0.25, 0.41, 0.88}

\definecolor{resultbest}{HTML}{CFE3ED}
\definecolor{resultsecond}{HTML}{EDF4F8}
\newcommand{\cellbest}{\cellcolor{resultbest}}
\newcommand{\cellsecond}{\cellcolor{resultsecond}}
\newcommand{\ranklegend}{\colorbox{resultbest}{first} and \colorbox{resultsecond}{second}}
\newcommand{\piold}{\pi_{\mathrm{old}}}
\newcommand{\piref}{\pi_{\mathrm{ref}}}

\newcommand{\pitheta}{\pi_{\theta}}

\newcommand{\vref}{\bm{v}_{\mathrm{ref}}}
\newcommand{\vold}{\bm{v}_{\mathrm{old}}}
\newcommand{\vtgt}{\bm{v}_{\mathrm{tgt}}}
\newcommand{\vtgti}{\bm{v}^{i}_{\mathrm{tgt}}}

\newcommand{\rhotheta}{\rho_{\theta}}

\colorlet{metablue}{blue!60!green}
\hypersetup{
    colorlinks,
    linkcolor={metablue},
    citecolor={metablue},
    urlcolor={metablue}
}

\usepackage{tikz}
\usetikzlibrary{positioning}

\usepackage{color,soul}

\usepackage{pifont}
\usepackage{ifsym}
\colorlet{myred}{red!85!black}
\usepackage{enumitem} % for itemize leftmargin
\usepackage[capitalise]{cleveref}

\newcommand{\cL}{\mathcal{L}}

\renewcommand{\E}{\operatorname{\mathbb{E}}}

\newcommand{\normal}{\mathcal{N}}
\newcommand{\unif}{\operatorname{Unif}}
\newcommand{\x}{\bm{x}}
\renewcommand{\v}{\bm{v}}
\newcommand{\I}{\bm{I}}
\newcommand{\zero}{\bm{0}}
\newcommand{\beps}{\bm{\epsilon}}

\title{Scaling Reinforcement Learning for Diffusion Models via Velocity Matching}

\author{%
\textbf{
Jaemoo Choi$^1$, \quad
Wei Guo$^1$, \quad
Yuchen Zhu$^1$, \quad
}
\\
\textbf{
Arash Vahdat$^2$, \quad
Molei Tao$^1$, \quad
Julius Berner$^2$, \quad
Yongxin Chen$^{1,2}$
}
\\
$^1$Georgia Institute of Technology, $^2$NVIDIA
\\
\texttt{\{jchoi843, wei.guo, yzhu738, mtao, yongchen\}@gatech.edu}
\\
\texttt{\{avahdat, jberner, yongxinc\}@nvidia.com}
}

\iclrfinalcopy % Uncomment for camera-ready version, but NOT for submission.
\DeclareRobustCommand{\revblue}[1]{{\color{black}\hypersetup{linkcolor=black,citecolor=black,urlcolor=black}#1}}
\DeclareRobustCommand{\revred}[1]{{\color{black}\hypersetup{linkcolor=black,citecolor=black,urlcolor=black}#1}}

\begin{document}

\maketitle
\begin{abstract}
    Reward-based fine-tuning of diffusion models has largely inherited likelihood-based policy optimization developed for autoregressive large language models (LLMs). Diffusion models, however, are natively trained through velocity regression and do not directly provide the likelihood of a generated sample.
    Existing methods address this using transition likelihoods along stochastic denoising trajectories or evidence lower bounds (ELBOs) to approximate generated-sample likelihoods.
    These approximations arise from applying likelihood-based updates to models whose native training and generation operate through velocity fields. We instead take velocity matching as the starting point for reward fine-tuning. We propose \textbf{reward-based velocity matching (RVM)}, which weights the velocity-matching loss by reward with an additional anchor regression term, requiring neither likelihood estimation nor likelihood ratios.
    RVM recovers Reinforce Adjoint Matching at the update level and contains DiffusionNFT as a special case, while ELBO-based methods are closely related through the same velocity-regression structure. This unified view isolates reward design and anchor velocity as principal design choices in velocity-based fine-tuning.
    Across text-to-image, text-to-video, and image-to-video generation, RVM matches or outperforms the evaluated trajectory-based methods. On Wan2.1-T2V-1.3B, it achieves the highest VBench Overall among the evaluated methods, with substantially lower estimated training costs than the trajectory-based approaches. For video fine-tuning, we introduce a dynamic-tracking reward that provides explicit motion feedback and improves both Dynamic Degree and overall VBench performance.
    \begin{center}
    {\textbf{Project page}: \url{https://jaemoo-choi.github.io/RVM/}}
    \end{center}
\end{abstract}

\section{Introduction}
Reward-based fine-tuning for diffusion and flow models has become an important approach for adapting them to human preferences~\citep{wu2023human, xu2023imagereward} and downstream criteria such as semantic alignment, visual quality, motion, and physical plausibility~\citep{wang2026tagrpoboostinggrpoimagetovideo, tang2026vgrpoonlinereinforcementlearning, xue2026systematicposttrainframeworkvideo, li2026rethinkingrewardsignalsvideo}. 
A dominant approach transfers likelihood-based policy optimization developed for autoregressive large language models (LLMs), such as PPO and GRPO~\citep{schulman2017proximal, shao2024deepseekmath}, to diffusion models.
For autoregressive LLMs, likelihood-based policy optimization is natural because sequence likelihoods factor into tractable next-token probabilities. However, diffusion models are fundamentally parameterized and trained through velocity matching, while the likelihood of a generated sample is not directly available. 

\begin{figure}[t]
    \centering
    \includegraphics[width=\linewidth]{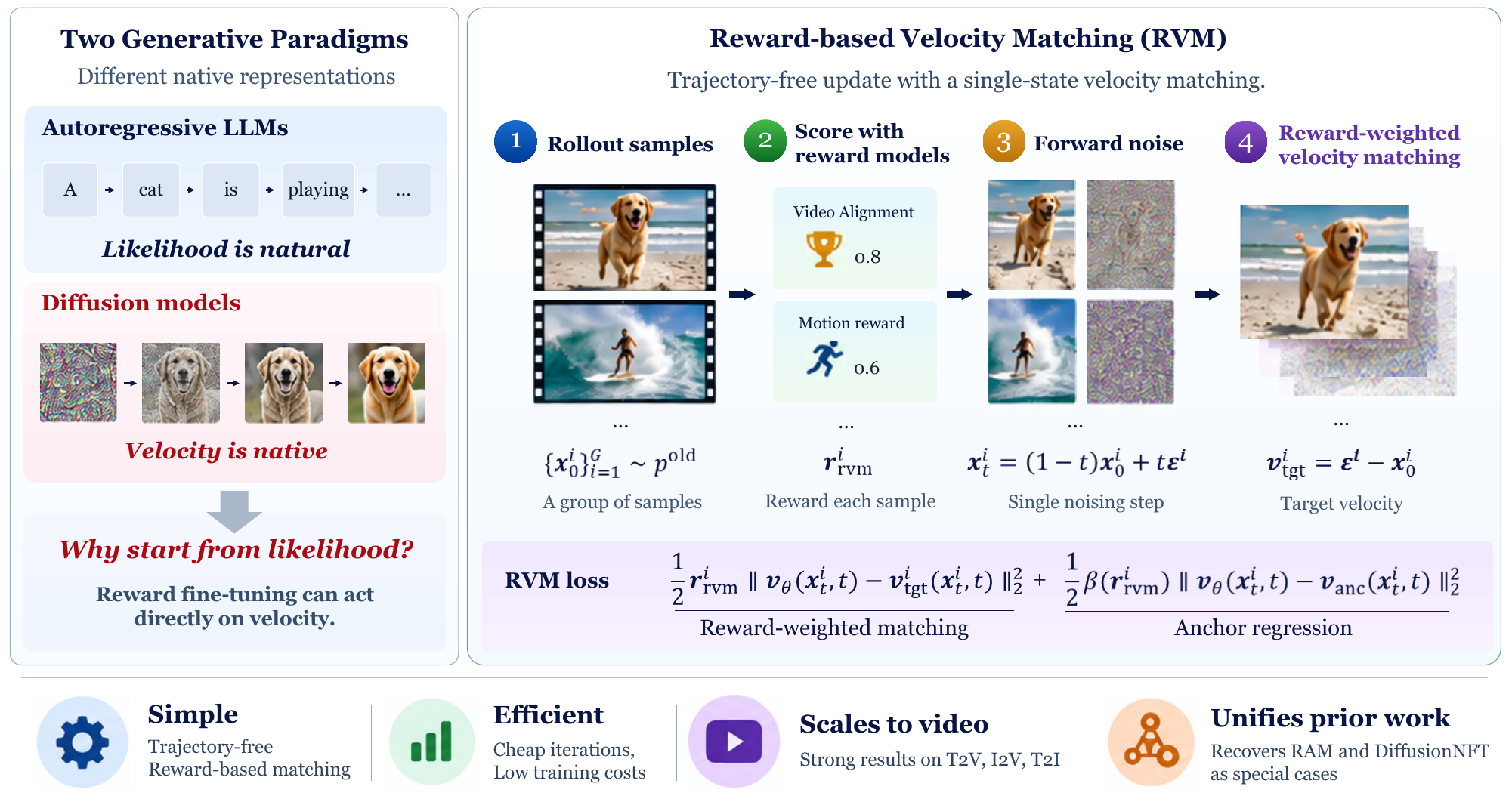}
    \caption{\textbf{Overview of reward-based velocity matching (RVM).} \textbf{Left:} Autoregressive LLMs operate naturally through token probabilities, whereas diffusion models are trained through velocity matching, motivating reward fine-tuning directly in velocity space. \textbf{Right:} RVM generates samples, scores them with reward models, and forward-noises each to an intermediate state. RVM objective combines reward-weighted velocity matching with anchor regression, requiring neither likelihood-ratio estimation nor stored rollout trajectories. This formulation supports different rollout samplers, applies to image and video generation, and recovers RAM and DiffusionNFT as special cases.}  
    \label{fig:teaser}
\end{figure}

Existing methods address this issue by reconstructing likelihood-based policy updates indirectly. \textit{Trajectory-based} methods use tractable one-step transition probabilities $\pi_\theta(\bm{x}_{t-\Delta t}\mid \bm{x}_t)$ along stochastic denoising trajectories \citep{liu2025flow, xue2025dancegrpo}. \textit{ELBO-based} methods approximate the log-likelihood of a generated sample $\bm{x}_0$ using an evidence lower bound (ELBO) as a surrogate, computed from velocity-matching losses at independently noised states $\bm{x}_t$ \citep{xue2025advantage, choi2026rethinkingdesignspacereinforcement}. 
Constructing $\bm{x}_t$ directly from the generated sample $\bm{x}_0$ decouples training updates from the rollout sampler, allowing fast deterministic ODE sampling.
These ELBO-based policy-gradient updates can be expressed as weighted velocity matching, with weights given by the product of the reward and an estimated likelihood ratio. 
In practice, these likelihood-ratio estimates are biased and have high variance. % \av{biased and high-variance}
More fundamentally, the need for such approximations arises from applying likelihood-based policy optimization to models whose native training and generation operate through velocity fields. This mismatch raises our central question:

\graybox{
\vspace{1.0em}
{\leftskip=8pt \rightskip=8pt
\emph{Can we formulate reward fine-tuning directly through velocity matching, without reconstructing the objective from likelihood-based policy optimization?}\par}
\vspace{0.4em}
}

We argue that reward fine-tuning can act directly on the native velocity representation of diffusion
models, without reconstructing from likelihood estimation. We propose \textbf{reward-based velocity matching (RVM)}, an objective that weights the velocity-matching loss by reward with an additional anchor regression term.
As illustrated in \cref{fig:teaser}, given a generated sample $\bm{x}_0^i$, RVM independently construct noise state $\bm{x}_t^i$, and its velocity target $\bm{v}_{\mathrm{tgt}}^i=\bm\epsilon^i - \bm{x}^i_0$, then weights its velocity-regression loss according to the sample reward $r^i_\mathrm{rvm}$.
Positive reward weights reinforce the velocity direction associated with the generated sample, while negative weights suppress it. The anchor term regularizes the predicted velocity toward a anchor field $\bm{v}_{\mathrm{anc}}$, such as the velocity predicted by the frozen pretrained model. This penalizes excessive deviation from the reference model's denoising behavior during reward fine-tuning.
% \av{Providing some intuition or examples on what the anchor field is would be helpful. I assume this can be the reference model, for example. By calling it "anchor" I don't have a mental picture on it.} \jaemoo{Revised!}

Recent methods such as RAM~\citep{bergmeister2026reinforceadjointmatchingscaling} and DiffusionNFT~\citep{zheng2025diffusionnft} have demonstrated strong empirical performance in reward fine-tuning, motivated by stochastic control and contrastive learning, respectively. RVM unifies these methods, together with ELBO-based methods such as AWM~\citep{xue2025advantage}, within a weighted velocity matching framework. Specifically, RVM recovers their objectives at the level of updates under appropriate choices of sampling policy, reward weighting, and anchor, revealing a shared regression structure beneath their distinct motivations.
% \yongxin{There exist several methods that also ....}
This unified view organizes objective design around how rewards are defined and incorporated, and how the anchor field and its coefficient are chosen.

We evaluate the scalability and efficiency of RVM across text-to-image, text-to-video, and image-to-video generation using Stable Diffusion 3.5-Medium~\citep{esser2024scaling}, Wan2.1-T2V-1.3B~\citep{wan2025wan}, and SkyReels-I2V~{\color{black}\hypersetup{citecolor=black}\citep{SkyReelsV1}}, respectively. RVM outperforms the evaluated trajectory-based methods in generation quality. Moreover, RVM substantially reduces both rollout and gradient-update costs relative to the evaluated trajectory-based methods. We also find that RVM achieves performance on par with or better than recent methods such as DiffusionNFT and RAM, whose objectives can be recovered within our unified framework. We further analyze the effects of reward and anchor choices. Preference and visual-quality rewards can favor nearly static videos, whereas adding our dynamic-tracking (DT) reward improves motion and overall VBench performance.
% Anchoring to a reference velocity field with classifier-free guidance (CFG) improves semantic alignment.
Overall, our \textbf{main contributions} are:
\begin{itemize}[leftmargin=*]
    \item We introduce \textbf{RVM}, which formulates reward fine-tuning directly in the native velocity representation of diffusion models. RVM combines reward-weighted velocity matching with an anchor regression term, requiring neither likelihood estimation nor likelihood ratios.
    % \yongxin{emphasize more diffusion model native, velocity matching} We introduce RVM, a trajectory-independent objective that applies reward feedback directly to velocity matching, avoiding likelihood-based policy-ratio estimation.
    \item Extensive experiments demonstrate that RVM scales across text-to-image, text-to-video, and image-to-video generation. Its trajectory-free updates and velocity matching loss enable efficient fine-tuning. RVM achieves competitive or superior performance across tasks, surpassing all evaluated trajectory- and velocity-based baselines in VBench Overall on Wan2.1-T2V-1.3B.
    % \yongxin{efficiency? make it more exciting}
    \item We show that recent competitive methods, such as RAM, DiffusionNFT, and the ELBO-based approaches can be unified within the RVM framework. This unified view identifies reward design and anchor velocity as principal design choices within this family of objectives.
    \item We analyze the effects of reward design and anchor velocity. We introduce a dynamic-tracking reward that improves motion and overall video quality. The evaluated anchor configurations show different balances between overall quality and semantic alignment.
\end{itemize}
 %\av{AWM connection was not discussed in details.}\jaemoo{Revised Sec. 3.2.}

\section{Background}

\subsection{Diffusion and Flow Models}
Diffusion and flow-based models \citep{song2020score, lipman2023flow, albergo2023stochastic, liu2023flow} learn a data distribution on $\R^{d}$ by interpolating clean data $\x_0 \sim \pdata$ and an independent noise $\beps\sim\normal(\zero,\I)$ through $\x_t=(1-t)\x_0+t\beps$, or equivalently, $\x_t|\x_0\sim\normal((1-t)\x_0,t^2\I)$. 
The model learns a velocity field $\v_\theta(\x_t,t)$ through a weighted regression: $$\min_\theta\E_{t,\x_0,\beps} \left[ w(t)\|\v_\theta(\x_t,t)-\vtgt\|^2 \right],$$ where $t\sim\unif(0,1)$, $\vtgt:=\beps-\x_0$, and $w(t)$ is a positive weight function.
To generate samples from the model, the standard practice is to simulate the ODE $\dot\x_t=\v_\theta(\x_t,t)$ from $\x_1\sim\normal(\zero,\I)$ to $\x_0$, or use the equivalent SDE samplers.

\subsection{Reward Fine-tuning and Policy-Gradient Objectives}

\paragraph{Reward fine-tuning via policy gradients}
Following RL fine-tuning of LLMs, reward fine-tuning treats a generative model as a policy $\pitheta$ and optimizes reward $r$ through likelihood-based policy gradients such as REINFORCE \citep{sutton1999policy}. Given samples from a sampling policy $\piold$, a standard importance-weighted policy-gradient objective can be written as
\begin{align}\label{eq:reinforce}
     \min_{\theta} \; - \mathbb{E}_{\bm{x} \sim \piold} \big[ r(\bm{x}) \bar\rho_\theta(\bm{x}) \log \pitheta(\bm{x}) \big] + \beta \operatorname{KL}(\pi_{\theta} \| \piref), \quad \bar{\rho}_{\theta}(\bm{x})
:= \frac{\operatorname{sg}[\pitheta(\bm{x})]}{\piold(\bm{x})}, \tag{REINFORCE}
\end{align}
where $\bar\rho_\theta$ is the detached importance ratio and $\beta$ controls the regularization strength. The corresponding optimal policy satisfies $\pi^\star (\bm{x}) \propto \piref(\bm{x})\exp(r(\bm{x})/\beta)$. 
When conditioning on $c$, both the policy and reward depend on it,
i.e., $\bm{x}\sim\pitheta(\cdot\mid c)$ and $r=r(\bm{x},c)$; we omit $c$ throughout for brevity.
GRPO further uses group-relative rewards and clips the likelihood ratio.
Given a group of $G$ samples $\{\bm{x}^i\}_{i=1}^G$, its standardized reward is $r^i_{\mathrm{grpo}} = \frac{r(\bm{x}^i) - \operatorname{mean}_{g} r(\bm{x}^g)}{\operatorname{std}_{g} r(\bm{x}^g)}$.

\paragraph{Trajectory-based approaches}
Because diffusion endpoint likelihoods, i.e., $\pitheta(\bm{x}_0)$, and their ratios are intractable, trajectory-based methods instead use tractable one-step transition probabilities $\pitheta(\bm{x}_{t-\Delta t}\mid\bm{x}_t)$ along stochastic denoising trajectories~\citep{black2023training, fan2023dpok}. FlowGRPO~\citep{liu2025flow} and DanceGRPO~\citep{xue2025dancegrpo} apply GRPO-style updates to these transitions, requiring costly SDE rollouts and updating toward stochastic intermediate targets, which increases computational cost and can introduce additional variance.

\paragraph{ELBO-based approaches}
\citet{xue2025advantage, choi2026rethinkingdesignspacereinforcement} estimate the sample log-likelihood $\pitheta$ and importance ratio $\bar\rho_\theta$ in \cref{eq:reinforce} using a velocity-regression surrogate,
\begin{align}\label{eq:elbo}
\log\pi_{\{\theta,{\rm old}\}}(\bm{x}_0)
\approx
-\mathbb{E}_{t, \bm{\epsilon}}
\big[
    w(t)\|\v_{\{\theta,{\rm old}\}}(\bm{x}_t,t)-\vtgt\|_2^2
\big], \tag{ELBO}
\end{align}
where $\bm{x}_t = (1-t)\,\bm{x}_0 + t\,\bm{\epsilon}$ and $\vtgt = \bm{\epsilon} - \bm{x}_0$.
ELBO-based approaches estimate $\bar\rho_\theta$ by exponentiating the difference between the ELBO estimates for the current and sampling policies in \cref{eq:elbo}. Although effective in practice, this additional estimation step can make the resulting velocity updates sensitive to estimation noise.
One variant is Advantage Weighted Matching (AWM)~\citep{xue2025advantage}, the on-policy version of this update, with $\piold:=\pitheta$.

% Accurately estimating the likelihood ratio $\bar{\rho}_\theta$ in \cref{eq:epg2} through the ELBO requires averaging the velocity-matching loss over diffusion times $t$ and noise $\bm{\epsilon}$, which can be computationally costly for large models. 
% At the same time, after ELBO substitution, the resulting policy-gradient update already takes a form closely related to reward-weighted velocity regression. This motivates bypassing likelihood estimation altogether and applying reward feedback directly to the native velocity-regression objective of diffusion models, which we refer to as \textit{reward-weighted velocity matching}.

\begin{figure}[t]
    \centering
    \captionsetup[subfigure]{skip=0pt}
    \includegraphics[width=0.96\linewidth]{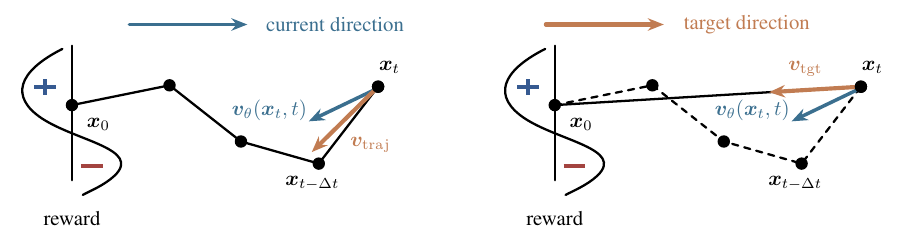}\par\vspace{-0.6em}
    \subcaptionbox{Trajectory-based.}[0.46\linewidth]{}
    \hspace{0.02\linewidth}
    \subcaptionbox{Velocity-based.}[0.46\linewidth]{}
    \caption{We illustrate both updates under the same rollout and compare their loss targets. Given a generated endpoint $\bm{x}_0$ with \textit{positive reward}, (a) the \textit{trajectory-based} loss uses the target $\bm{v}_{\mathrm{traj}}$ determined by the realized noisy transition from $\bm{x}_t$ to $\bm{x}_{t-\Delta t}$, whereas (b) \textit{velocity-based} losses, such as RVM, use the endpoint-defined target $\vtgt=\bm{\epsilon}-\bm{x}_0$. Arrows indicate denoising directions. Negative rewards reverse the corresponding attraction into repulsion. See \cref{app:traj_vel} for details.}
    \label{fig:traj_vs_vel}
\end{figure}

\section{Reward Fine-Tuning in Velocity Space}
\label{sec:method}

% In this section, we present a simple velocity-matching recipe for reward fine-tuning of video diffusion models. In \cref{sec:bg-objectives}, we introduce our objective alongside existing velocity-based losses. In \cref{sec:newloss}, we show how these losses connect. We then turn to reward shaping in \cref{subsec:ablation}, which our analysis identifies as the factor that governs performance.

% \paragraph{From likelihood estimation to reward-based velocity matching.}

\subsection{Reward-Based Velocity Matching}
\label{sec:bg-objectives}
\revblue{The likelihood-based approaches above ask how to adapt likelihood-based policy gradients to diffusion models. We step back and ask whether likelihood should be the starting point at all. As illustrated in \cref{fig:teaser} (left), likelihood-based optimization naturally fits autoregressive LLMs, whose native outputs are token probabilities. Diffusion models, however, learn to generate through velocity fields, suggesting that reward feedback can act directly on this representation. This motivates \textbf{reward-based velocity matching (RVM)}: we formulate reward fine-tuning directly as velocity regression, removing the need to reconstruct likelihood-based policy updates.}
% The likelihood-based approaches above require surrogate formulations or approximations to apply policy optimization to diffusion models. This motivates a more direct question: can reward fine-tuning be formulated through the model’s native velocity-matching objective, following pretraining? We take this perspective and introduce reward-based velocity matching (RVM), which incorporates reward feedback directly into velocity regression. 
% \yongxin{this paragraph can be made more exciting. don't assume they have gone through the introduction}

\paragraph{Direct velocity-space objective}
\revblue{We formulate reward fine-tuning directly in velocity space using two regression terms: reward-weighted velocity matching to incorporate reward feedback, and anchor regression to control model drift. As illustrated in \cref{fig:teaser} (right), the update follows four steps: generate samples $\bm{x}_0^i$ from sampling policy $\piold$, score them with reward models $r^i_{\mathrm{rvm}}$, forward-noise them to states $\bm{x}_t^i=(1-t)\bm{x}_0^i+t\bm{\epsilon}^i$, where $\bm{\epsilon}^i\sim\mathcal{N}(\bm{0},\bm{I})$, and perform reward-weighted velocity regression with targets $\vtgti=\bm{\epsilon}^i-\bm{x}_0^i$ and anchor field $\bm{v}_{\mathrm{anc}}$. These steps yield the following \emph{reward-based velocity matching} (RVM) objective:}

\graybox{%
\begin{align}
\cL_{\mathrm{rvm}}
=
\tfrac{1}{2} \ \mathbb{E}_{\bm{x}^{1:G}_0 \sim \piold}
\left[
r^i_{\mathrm{rvm}}\,
\|\vtheta(\bm{x}_t^i,t)-\vtgti\|_2^2
+
\beta (r^i_\mathrm{rvm}) \,\|\vtheta(\bm{x}_t^i,t)-\bm{v}_{\mathrm{anc}}(\bm{x}_t^i,t)\|_2^2
\right], \tag{RVM}\label{eq:rvm}
\end{align}
}

\revblue{where $\beta:\R\to\R$ controls the strength of the anchor regularization.
The first term incorporates reward feedback by weighting regression toward $\vtgti$ with $r^i_{\mathrm{rvm}}$: positive weights pull the predicted velocity toward the target, while negative weights push it away.
The second term regularizes the prediction toward an \emph{anchor field} $\bm{v}_{\mathrm{anc}}$, such as the velocity predicted by a frozen pretrained model. When $\beta(r^i_{\mathrm{rvm}})>0$, this term penalizes deviations from the anchor's denoising behavior and thereby controls model drift.
As shown in \cref{fig:teaser} (right), both terms operate at the same independently forward-noised state, requiring neither likelihood-ratio estimation nor stored denoising trajectories.
These \emph{trajectory-free updates} decouple training from the rollout sampler, allowing samples to be generated using different samplers, including fast deterministic ODE solvers.}
% \yongxin{anchor? also refer to teaser}

\paragraph{Reduced target variance and sampler-agnostic training} \cref{fig:traj_vs_vel} illustrates a source of high target variance in trajectory-based methods: their regression targets depend on the noisy intermediate state $\bm{x}_{t-\Delta t}$ sampled during an SDE rollout. Consequently, a positive reward reinforces the direction toward this intermediate sample, which can deviate from the direction toward the rewarded endpoint $\bm{x}_0$. RVM instead constructs its target directly as $\vtgt=\bm{\epsilon}-\bm{x}_0=(\bm{x}_t-\bm{x}_0)/t$ for $t>0$, so the denoising direction $-\vtgt$ points toward $\bm{x}_0$. This removes the target variability introduced by sampled reverse transitions (\cref{app:traj_vel}). The same construction also makes loss evaluation independent of the rollout sampler: RVM requires only the generated endpoint and independently forward-noised states. It therefore supports both stochastic SDE and deterministic ODE rollouts without retaining intermediate trajectories, enabling flexible and memory-efficient reward fine-tuning.

% \paragraph{\revblue{Well-posed updates with negative rewards}}
% \revblue{Anchoring can make the population RVM objective well posed even when individual reward weights are negative. Under the bounded-weight and finite-moment assumptions in \cref{app:proximal_rvm}, anchoring to the detached pre-update field with constant $\beta>0$ gives a unique $L^2$ proximal update if $\beta+\mathbb{E}_{\piold}[r^i_{\mathrm{rvm}}\mid\bm{x}_t^i,t]$ has a strictly positive uniform lower bound. In the special case $\beta=0$ with exponential sample-wise reward weights, exact population matching fits the conditional velocity of the KL-regularized reward optimum (\cref{app:reward_reweighting}).}
% \jaemoo{Revision Required} \yongxin{remove or move to appendix}

\subsection{A Unified View via Reward-based Velocity Matching}\label{sec:newloss}
% \yongxin{mention these two have superior performance}
RAM~\citep{bergmeister2026reinforceadjointmatchingscaling} and DiffusionNFT~\citep{zheng2025diffusionnft} have demonstrated strong empirical performance in reward fine-tuning. Although motivated by control-theoretic and contrastive-learning perspectives, respectively, both admit RVM representations under specific choices of sampling policy, reward, and anchor.
Moreover, ELBO-based policy-gradient approaches such as \cref{eq:epg2}~\citep{choi2026rethinkingdesignspacereinforcement} and AWM~\citep{xue2025advantage} share a common velocity matching structure.

\paragraph{RAM as an RVM update}
RAM regresses toward a detached reward-dependent target,
\begin{align}
    \cL_{\mathrm{ram}}
    =\tfrac{1}{2} \
    \mathbb{E}_{\bm{x}^{1:G}_0\sim\pitheta}
    \left[
    \left\|
    \vtheta(\bm{x}_t^i,t)
    -
    \operatorname{sg}
    \left[
    \bm{v}_{\mathrm{ref}}(\bm{x}_t^i,t)
    +
    r^i_{\mathrm{ram}}\big(\vtgti-\vtheta(\bm{x}_t^i,t)\big)
    \right]
    \right\|_2^2
    \right],
    \tag{RAM}
\label{eq:ram}
\end{align}
whose gradient with respect to the predicted velicty $\bm{v}_\theta$ is \footnote{We omit expectations in the gradient for simplicity.}
\begin{align}
\nabla_{\vtheta} \cL_{\mathrm{ram}} = \vtheta - \operatorname{sg}\!\big[\vref + r^i_{\mathrm{ram}}(\vtgti-\vtheta)\big] = r^i_{\mathrm{ram}}\,(\vtheta-\vtgti) + (\vtheta-\vref),
\label{eq:ram-grad-main}
\end{align}
which exactly matches the RVM gradient $r^i(\vtheta-\vtgti)+\beta(\vtheta-\bm{v}_{\mathrm{anc}})$ with the reference anchor $\bm{v}_{\mathrm{anc}}=\vref$ at unit strength $\beta=1$ and reward weight $r^i_\mathrm{rvm}=r^i_{\mathrm{ram}}$. Thus, RAM is update-equivalent to on-policy RVM ($\piold=\pitheta$) with the frozen reference as its anchor.

\paragraph{DiffusionNFT as a special case of RVM}
The DiffusionNFT loss is\footnote{The factor $1/\bar\beta$ is a constant rescaling of the original DiffusionNFT~\citep{zheng2025diffusionnft} objective.}
\begin{align}
\cL_{\mathrm{nft}}
= \tfrac{1}{2\bar\beta}
\mathbb{E}_{\bm{x}^{1:G}_0\sim\piold}
\left[
% \frac{1}{G}
% \sum_{i=1}^{G}
r^i_{\mathrm{nft}}
\|\bm{v}^{+}_\theta(\bm{x}_t^i,t)-\vtgti\|_2^2
+
(1-r^i_{\mathrm{nft}})
\| \bm{v}^{-}_\theta(\bm{x}_t^i,t)-\vtgti\|_2^2
\right].
\tag{NFT}
\label{eq:nft}
\end{align}
where $\bm{v}^{\pm}_\theta=\bm{v}_{\mathrm{anc}}\pm\bar\beta(\vtheta-\bm{v}_{\mathrm{anc}})$ are the positive and negative velocity branches, $\bm{v}_{\mathrm{anc}}$ is an EMA velocity, $\bar\beta>0$ is the guidance strength, and $r^i_{\mathrm{nft}}\in[0,1]$ is the reward weight. Expanding the squared norm yields
\begin{align}
\cL_{\mathrm{nft}} = \tfrac{r^i_{\mathrm{rvm}}}{2}\,\big\|\vtheta-\vtgti\big\|_2^2 + \tfrac{\bar\beta-r^i_{\mathrm{rvm}}}{2}\,\big\|\vtheta-\bm{v}_{\mathrm{anc}}\big\|_2^2 + \mathrm{const}, \qquad r^i_{\mathrm{rvm}} := 2r^i_{\mathrm{nft}}-1.
\label{eq:nft-rvm-main}
\end{align}
Thus, \cref{eq:nft} objective equals \cref{eq:rvm} up to constants for the EMA anchor and$\beta(r^i_{\mathrm{rvm}})=\bar\beta-r^i_{\mathrm{rvm}}$. 

% \graybox{
% \vspace{0.1in}
% \begin{theorem}[Gradient-level Unification of Velocity-based Reward Objectives]\label{thm:unify}
% Fix a group of samples $\{\bm{x}^i_0\}^G_{i=1}\sim \piold$ and omit the shared argument $(\bm{x}_t^i,t)$ on all velocity fields. Then, RAM and DiffusionNFT match RVM under specific choices of the sampling policy $\piold$, anchor velocity $\bm{v}_\mathrm{anc}$, reward $r^i_\mathrm{rvm}$, and anchor strength $\beta(r)$:
% \begin{itemize}[leftmargin=*]
% \item \textbf{RAM is update-equivalent to RVM.} With sampling policy $\piold=\pitheta$, anchor $\bm{v}_{\mathrm{anc}}=\vref$, strength $\beta\equiv 1$, and reward $r^i_\mathrm{rvm}=r^i_{\mathrm{ram}}$,
% \begin{align*}
%     \nabla_{\vtheta}\cL_{\mathrm{ram}} = \nabla_{\vtheta}\cL_{\mathrm{rvm}}.
% \end{align*}
% \item \textbf{DiffusionNFT is a special case of RVM.} With sampling policy $\piold$, the EMA anchor $\bm{v}_{\mathrm{anc}}$, reward weight $r^i_\mathrm{rvm} = 2r^i_{\mathrm{nft}}-1$, and the reward dependent anchor strength $\beta(r^i_\mathrm{rvm}) = \bar\beta - r^i_\mathrm{rvm}$,
% \begin{align*}
% \cL_{\mathrm{nft}} = \cL_{\mathrm{rvm}} + \mathrm{const}, \qquad \nabla_{\vtheta} \cL_{\mathrm{nft}} = \nabla_{\vtheta} \cL_{\mathrm{rvm}}.
% \end{align*}
% \end{itemize}
% \end{theorem}
% }

% Queue the qualitative comparison early enough for page 6.
% Keep natural figure/text spacing through the main results pages.
\raggedbottom
\begin{figure*}[t]
    \centering
    \includegraphics[width=0.92\linewidth]{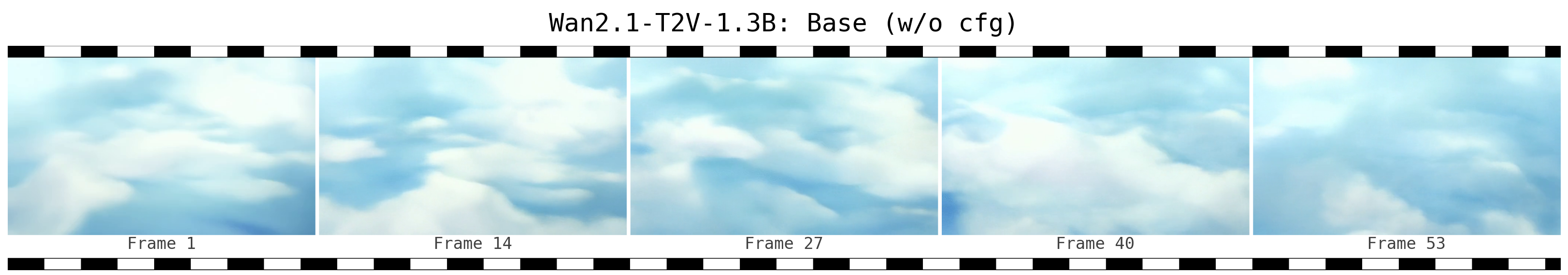}\\[0pt]
    \includegraphics[width=0.92\linewidth]{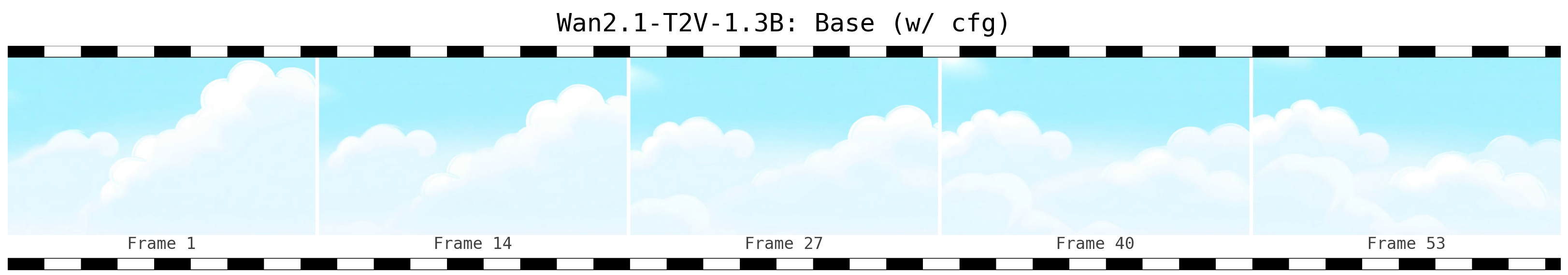}\\[0pt]
    \includegraphics[width=0.92\linewidth]{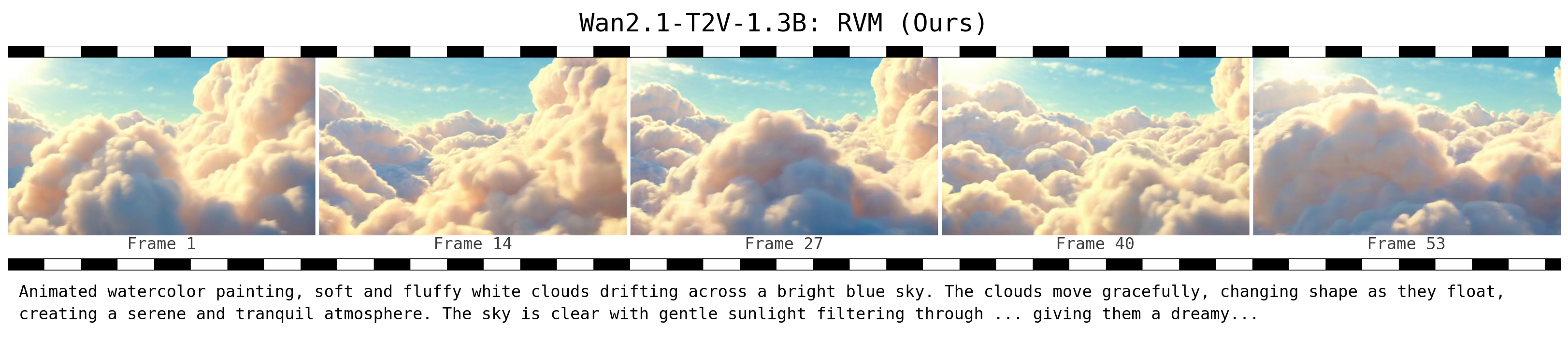}\\[3pt]
    \includegraphics[width=0.92\linewidth]{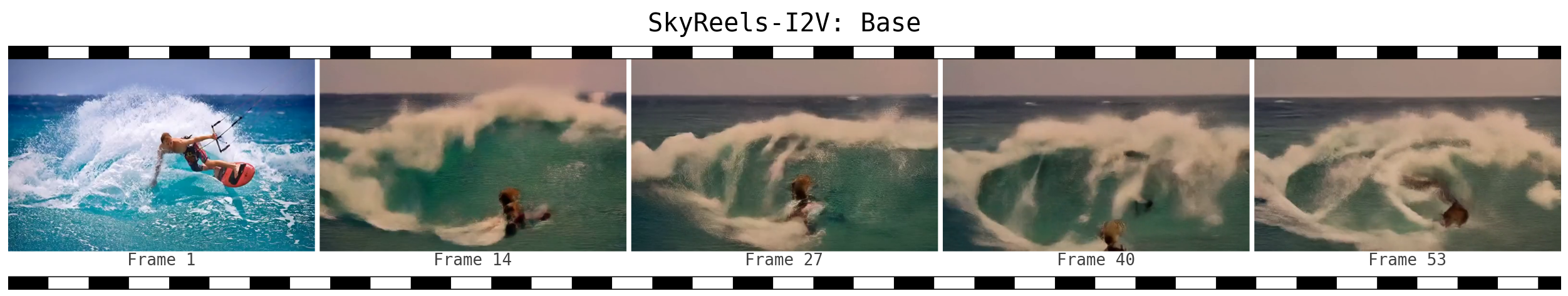}\\[0pt]
    \includegraphics[width=0.92\linewidth]{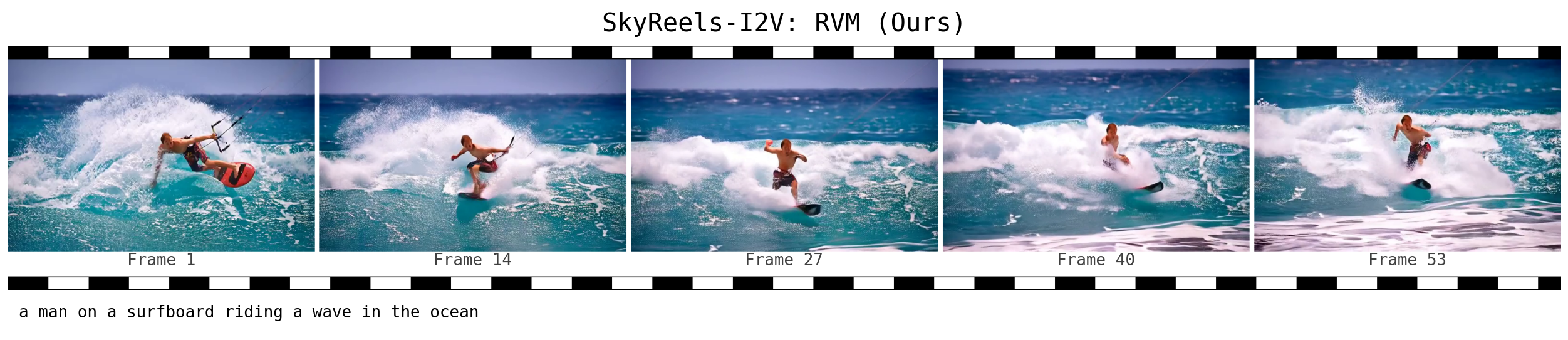}\\[-7pt]
    \caption{\textbf{Qualitative comparison between the base model and RVM} on Wan2.1-T2V-1.3B and SkyReels-I2V-V1. Additional examples are provided in \cref{app:qualitative}.}
    % \vspace{-10pt}
    \label{fig:all_comp}
\end{figure*}

% Tables 1--3 stay together above body text on page 7.
% Keep the figure queue to preserve the order of all results blocks.
\begin{figure*}[t]
\captionsetup{type=table,font=small,skip=4pt}
\centering
\caption{\textbf{Evaluation results on Wan2.1-T2V-1.3B.} We report reward scores including VideoAlign, HPSv3 and VBench on VBench dataset. \texttt{Dyn.\ Degree} is the VBench dynamic-degree dimension, one of the seven quality dimensions already averaged into \texttt{Quality}, reported on its own because it is the only one that measures how much the video moves. Blue shading marks the \ranklegend highest scores per column, including ties. $^\dagger$cited numbers from \citet{li2026rethinkingrewardsignalsvideo}.}
\label{tab:vbench_t2v}
\scriptsize
\setlength{\tabcolsep}{5pt}
\renewcommand{\arraystretch}{1}
\begin{tabular*}{\textwidth}{@{\extracolsep{\fill}}l cccc cccc}
\toprule
& \multicolumn{4}{c}{Reward} & \multicolumn{4}{c}{VBench} \\
\cmidrule(lr){2-5}\cmidrule(lr){6-9}
Method & TA & VQ & MQ & HPSv3 & Dyn. Degree & Quality & Semantic & Overall \\
\midrule
Wan2.1-1.3B (w/o CFG) & $2.53$ & $2.33$ & $0.35$ & $1.85$ & $54.17$ & $78.59$ & $65.73$ & $76.02$ \\
Wan2.1-1.3B (w/ CFG) & \cellbest$5.50$ & $4.08$ & $0.97$ & $7.85$ & \cellsecond$65.28$ & $83.72$ & \cellbest$80.64$ & \cellsecond$83.10$ \\
\midrule
\rowcolor{black!8}\multicolumn{9}{l}{\emph{Trajectory-based}} \\
FlowGRPO & $2.48$ & $2.24$ & $0.27$ & $1.62$ & $55.56$ & $78.67$ & $64.87$ & $75.91$ \\
DanceGRPO$^\dagger$ & $-1.17$ & $3.41$ & $0.30$ & -- & $58.33$ & $82.81$ & $66.06$ & $79.46$ \\
TaRoS$^\dagger$ & $-0.71$ & $3.50$ & $0.31$ & -- & $57.67$ & $83.23$ & $66.26$ & $79.84$ \\
TaRoS-72B$^\dagger$ & -- & -- & -- & -- & $58.33$ & $83.66$ & $68.15$ & $80.56$ \\
\midrule
\rowcolor{black!8}\multicolumn{9}{l}{\emph{Velocity-based}} \\
DiffusionNFT & $5.03$ & \cellsecond$5.04$ & \cellsecond$1.63$ & \cellsecond$9.70$ & $63.89$ & \cellsecond$84.37$ & \cellsecond$77.28$ & $82.95$ \\
RAM & $4.86$ & $4.23$ & $1.14$ & $8.39$ & $61.11$ & $83.35$ & $76.33$ & $81.95$ \\
Ours (RVM) & \cellsecond$5.12$ & \cellbest$5.47$ & \cellbest$1.86$ & \cellbest$10.90$ & \cellbest$75.00$ & \cellbest$86.09$ & $76.28$ & \cellbest$84.13$ \\
\bottomrule
\end{tabular*}
%\vspace{-1em}
\vspace{4pt}

\centering
\caption{\textbf{Evaluation results on SkyReels-I2V.} We report reward scores including VideoAlign, HPSv3 and VBench-I2V on VBench-I2V dataset. \texttt{Dyn.\ Degree} is the VBench dynamic-degree dimension, one of the seven quality dimensions already averaged into \texttt{Quality}. Blue shading marks the \ranklegend highest scores per column, including ties. DiffusionNFT and RAM are trained with the same reward mixture as ours.
% $^\S$The untrained base sampled with true classifier-free guidance (scale $6$) and the official SkyReels-V1 negative prompt; all fine-tuned rows sample without CFG. Even against this stronger base ($81.68 \to 82.12$ Overall), the fine-tuned gains remain, while CFG lowers motion (dynamic degree $64.63 \to 47.97$).
}
\label{tab:vbench_i2v}
\scriptsize
\setlength{\tabcolsep}{5pt}
\renewcommand{\arraystretch}{1}
\begin{tabular*}{\textwidth}{@{\extracolsep{\fill}}l cccc cccc}
\toprule
& \multicolumn{4}{c}{Reward} & \multicolumn{4}{c}{VBench-I2V} \\
\cmidrule(lr){2-5}\cmidrule(lr){6-9}
Method & TA & VQ & MQ & HPSv3 & Dyn. Degree & I2V & Quality & Overall \\
\midrule
{SkyReels-I2V-V1 (w/o CFG)} & $2.45$ & $2.30$ & $0.06$ & $3.18$ & \cellsecond$64.63$ & $87.79$ & $75.57$ & $81.68$ \\
{SkyReels-I2V-V1 (w/ CFG)} & $2.53$ & $1.99$ & $0.21$ & $4.56$ & $47.97$ & $88.87$ & $75.36$ & $82.12$ \\
\midrule
\rowcolor{black!8}\multicolumn{9}{l}{\emph{Trajectory-based}} \\
FlowGRPO & $1.67$ & $1.61$ & $-0.47$ & $-1.85$ & $41.46$ & $86.00$ & $69.93$ & $77.97$ \\
\midrule
\rowcolor{black!8}\multicolumn{9}{l}{\emph{Velocity-based}} \\
DiffusionNFT & $2.65$ & \cellbest$2.59$ & \cellbest$0.61$ & \cellsecond$6.43$ & $47.56$ & \cellsecond$93.18$ & $79.15$ & $86.16$ \\
RAM & \cellsecond$2.68$ & \cellsecond$2.56$ & \cellsecond$0.49$ & \cellbest$6.57$ & $51.22$ & \cellbest$93.49$ & \cellsecond$79.72$ & \cellbest$86.61$ \\
% Ours (RVM) & $2.52$ & $\mathbf{2.85}$ & $\mathbf{0.88}$ & $\mathbf{7.32}$ & 0.41 & \textbf{95.33} & 77.57 & \textbf{86.45} \\
Ours (RVM) & \cellbest$2.70$ & $2.45$ & $0.41$ & $6.03$ & \cellbest$72.36$ & $92.26$ & \cellbest$80.28$ & \cellsecond$86.27$ \\
\bottomrule
\end{tabular*}
%\vspace{-1em}
\vspace{4pt}

\centering
\caption{{\textbf{Evaluation results on SD3.5-M.} We report the OCR text-to-image fine-tuning task.
% , evaluated following \citet{choi2026rethinkingdesignspacereinforcement}. 
% \colorbox{gray!12}{gray} cells mark the rewards included in training. 
Blue shading marks the \ranklegend highest scores per column, including ties. $^\dagger$cited numbers from \citet{choi2026rethinkingdesignspacereinforcement}.}
Note that OCR, PickScore, ClipScore, HPSv2.1, are used in training.
}
\label{tab:t2i}
\scriptsize
\setlength{\tabcolsep}{5pt}
\renewcommand{\arraystretch}{1}
\begin{tabular*}{\textwidth}{@{\extracolsep{\fill}}l cccccc}
\toprule
Method & OCR & PickScore & ClipScore & HPSv2.1 & Aesthetic & ImgRwd \\
\midrule
{FlowGRPO$^\dagger$} & $0.92$ & $22.41$ & $0.290$ & $0.280$ & $5.32$ & $0.95$ \\
{AWM$^\dagger$} & $0.80$ & $20.70$ & $0.301$ & $0.206$ & $4.53$ & $-0.13$ \\
{DiffusionNFT$^\dagger$} & $0.93$ & $22.09$ & $0.307$ & $0.277$ & $5.17$ & \cellsecond$0.97$ \\
{PEPG$^\dagger$} & \cellsecond$0.94$ & \cellsecond$22.93$ & \cellbest$0.315$ & \cellsecond$0.302$ & \cellsecond$5.33$ & \cellbest$1.34$ \\
{Ours (RVM)} & \cellbest$0.95$ & \cellbest$23.03$ & \cellsecond$0.312$ & \cellbest$0.310$ & \cellbest$5.56$ & \cellbest$1.34$ \\
\bottomrule
\end{tabular*}
\end{figure*}

% The cost-comparison pair follows the tables on the next text page.
\begin{figure*}[t]
\captionsetup{font=small,skip=4pt}
    \centering
    \begin{minipage}[t]{0.47\linewidth}
        \vspace{0pt}
        \centering
        \includegraphics[width=\linewidth]{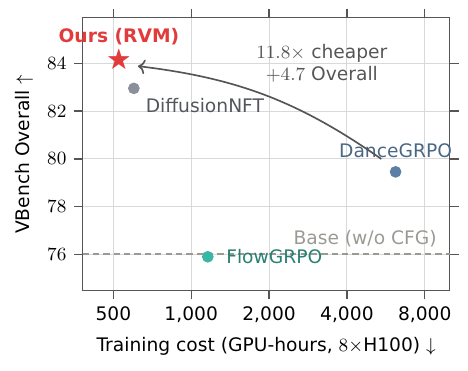}
        \vspace{-15pt}
        \caption{\textbf{Quality–-cost comparison.} RVM achieves a 4.7-point higher VBench score at $11.8\times$ lower training cost than DanceGRPO.}
        \label{fig:quality_cost}
    \end{minipage}\hfill
    \begin{minipage}[t]{0.50\linewidth}
        \vspace{0pt}
        \centering
        \includegraphics[width=\linewidth]{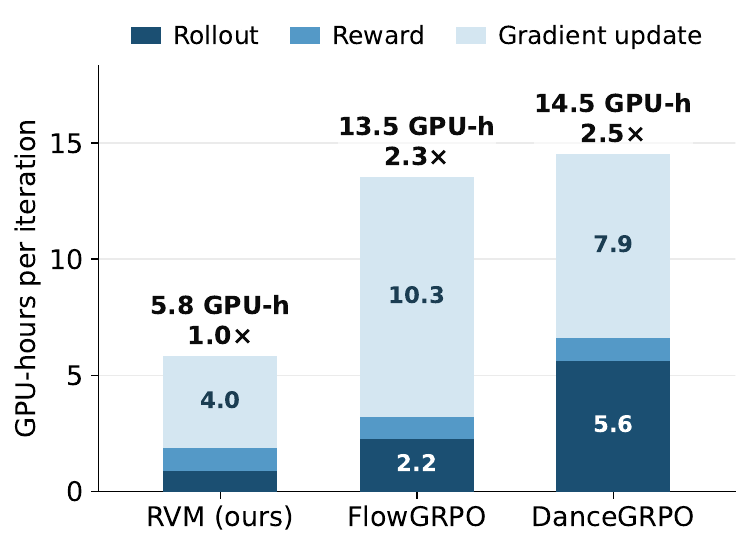}
        \vspace{-10pt}
        \caption{\textbf{Per-iteration GPU-hour cost on Wan2.1-1.3B.} Each iteration is decomposed into rollout, reward evaluation, and gradient update.}
        % Note that despite being the cheapest, RVM attains the highest VBench Overall ($84.13$ vs.\ $75.91$ for FlowGRPO).
        \label{fig:train_cost}
    \end{minipage}
\end{figure*}

\paragraph{Connection to ELBO-based policy gradients}
The ELBO-based approaches discussed above also admit a velocity-regression interpretation. Substituting the ELBO surrogate in \cref{eq:elbo} into the policy-gradient formulation in \cref{eq:reinforce}, together with reference regularization, gives the following surrogate for exact policy gradient (EPG)~\citep{choi2026rethinkingdesignspacereinforcement}:
\begin{align}
\mathcal{L}_{\mathrm{epg}}
=
\mathbb{E}_{\bm{x}^{1:G}_0 \sim \piold}
\big[  r^i_{\mathrm{epg}} \bar\rho_\theta(\bm{x}^{i}_0) \Vert \vtheta(\bm{x}^i_t, t) - \vtgti \Vert^2 + \beta \Vert \vtheta(\bm{x}^i_t, t) - \vref(\bm{x}^i_t, t)  \Vert^2 \big]. \tag{EPG}\label{eq:epg2}
\end{align}
Here, $r^i_{\mathrm{epg}}$ is the group-relative reward and $\bar\rho_\theta(\bm{x}_0^i)$ is a detached importance ratio estimated through the ELBO. This objective has the same weighted regression and anchor terms as RVM, but additionally weights the reward term by the estimated ratio. 

In the on-policy case $\piold=\pitheta$, the importance ratio satisfies $\bar\rho_\theta=1$, and \cref{eq:epg2} reduces to the AWM update~\citep{xue2025advantage}. RVM recovers this update, up to overall loss scaling, by setting $r^i_{\mathrm{rvm}}=r^i_{\mathrm{epg}}=r^i_{\mathrm{awm}}$, $\bm{v}_{\mathrm{anc}}=\vref$, and $\beta(r)\equiv\beta_{\mathrm{awm}}$, where $r^i_{\mathrm{awm}}$ and $\beta_{\mathrm{awm}}$ denote the AWM advantage weight and constant reference-regularization coefficient, respectively. This connection highlights the difference in starting point: RVM formulates these regression terms directly in velocity space, without reconstructing likelihood ratios.  

The following theorem summarizes the equivalances introduced in this section:

\graybox{
% Offset graybox's negative top inset and pad the full theorem heading.
\vspace{0.09in}
\begin{theorem}[Gradient-level Unification of Velocity-based Reward Objectives]
Using the notation and objective conventions of Sec.~3.1, consider gradients
at fixed training samples with reward weights and anchor velocities detached.
The following equivalences hold.

\begin{itemize}[leftmargin=*]
    \item \textbf{RAM is update-equivalent to RVM.}
    With $\piold=\pitheta$,
    $r^i_{\mathrm{rvm}}=r^i_{\mathrm{ram}}$,
    $\bm{v}_{\mathrm{anc}}=\vref$, and $\beta\equiv1$,
    \[
    \nabla_{\vtheta}\mathcal{L}_{\mathrm{ram}}=\nabla_{\vtheta}\mathcal{L}_{\mathrm{rvm}}.
    \]

    \item \textbf{DiffusionNFT is a special case of RVM.}
    With the EMA anchor,
    $r^i_{\mathrm{rvm}}=2r^i_{\mathrm{nft}}-1$, and
    $\beta(r)=\bar\beta-r$,
    \[
    \mathcal{L}_{\mathrm{nft}} =    \mathcal{L}_{\mathrm{rvm}}+\mathrm{const}, \qquad \nabla_{\vtheta}\mathcal{L}_{\mathrm{nft}}=\nabla_{\vtheta}\mathcal{L}_{\mathrm{rvm}},
    \]
    where the additive constant is independent of $\theta$.
    \item \textbf{AWM is update-equivalent to RVM.}
    With $\piold=\pitheta$,
    $r^i_{\mathrm{rvm}}=r^i_{\mathrm{awm}}$,
    $\bm{v}_{\mathrm{anc}}=\vref$, and
    constant $\beta$,
    \[
    \nabla_{\vtheta}\mathcal{L}_{\mathrm{awm}}=\nabla_{\vtheta}\mathcal{L}_{\mathrm{rvm}}.
    \]
\end{itemize}
\label{thm:unify}
\end{theorem}
% \vspace{4pt}
}

See \cref{subapp:connections} for the proof. Additionally, concurrent works are disucssed in \citep{xu2026designingreinforcementlearning, anonymous2026flowtrpo} in \cref{app:related}.
This unified view identifies reward weighting and anchor velocity as principal design choices that govern how reward feedback reshapes the learned velocity field.
We study their effects on generation behavior and quality in \cref{subsec:ablation}.
% where $\bm{x}^{1:G}_0 := \{\bm{x}^i_0\}_{i=1}^G$ denotes a group of samples drawn from the sampling policy $\piold$. For each sample, we construct
% $\bm{x}^i_t=(1-t)\bm{x}^i_0+t\bm{\epsilon}^i$, where
% $t\sim\mathrm{Unif}(0,1)$ and $\bm{\epsilon}^i\sim\mathcal{N}(0,I)$, with target velocity $\vtgti=\bm{\epsilon}^i-\bm{x}^i_0$.
% The weight $r^i_{\mathrm{epg}}$ is the group-relative reward of the $i$-th sample,
% $\vref$ is the frozen reference velocity field, and $\beta$ controls the strength
% of the reference regularization.
% The detached ELBO-based importance ratio weights the reward-regression term, while the second term
% penalizes deviation from the reference velocity.
%%%%%%%%%%%%%%%%%%%%%%%%%%%%%%%%%%%%%%%%%%%%%%%%%%%%%%%%%%%%%%%%%%%%%

\section{Experiments}
\label{sec:exp}
We evaluate reward fine-tuning across text-to-image (T2I), text-to-video (T2V), and image-to-video (I2V) generation. The experiments address two aspects of RVM. In \cref{subsec:main_results}, we compare RVM with trajectory-based, ELBO-based, and other velocity-based methods, and quantify sampling and update costs. In \cref{subsec:ablation}, we examine how reward mixtures and anchor velocities affect video quality, motion, and semantic alignment. We compare generation with fewer inference steps in \cref{app:few_step}.

\subsection{Experimental Setup}

\paragraph{Default configurations}
We fine-tune Stable Diffusion 3.5-Medium~\citep{esser2024scaling}, Wan2.1-T2V-1.3B~\citep{wan2025wan}, and SkyReels-V1-I2V~{\color{black}\hypersetup{citecolor=black}\citep{SkyReelsV1}} using LoRA~\citep{hu2022lora} and deterministic ODE sampling. Unless otherwise stated, video training uses RVM without an anchor or CFG, with a weighted combination of VideoAlign text alignment and motion quality, HPSv3, and dynamic-tracking (DT) rewards.
Full configurations including loss normalization and data protocols are provided in \cref{app:impl}. Details on video reward are elaborated in \cref{app:videoreward}.
% \yongxin{add key implementation details} 

% Keep the final main-text pages compact without changing text or figure sizes.
\begingroup
\setlength{\parskip}{3pt}
\paragraph{Evaluation}
We report reward-model scores alongside VBench~\citep{huang2024vbench} for T2V and VBench-I2V~\citep{huang2024vbenchpp} for I2V. T2I evaluation follows \citet{choi2026rethinkingdesignspacereinforcement} and reports OCR accuracy, image-quality metrics, and preference scores. Detailed protocols and complete video metrics are provided in \cref{app:impl,app:vbench_full}.

\paragraph{Baselines}
We compare RVM with pretrained models with and without CFG (scale $5$ for Wan2.1 and $6$ for SkyReels), trajectory-based FlowGRPO~\citep{liu2025flow}, DanceGRPO~\citep{xue2025dancegrpo}, and TaRoS~\citep{li2026rethinkingrewardsignalsvideo}, and velocity-based DiffusionNFT~\citep{zheng2025diffusionnft} and RAM~\citep{bergmeister2026reinforceadjointmatchingscaling}. Published DanceGRPO and TaRoS results use their respective protocols with $50$-step SDE sampling and CFG. T2I baselines additionally include ELBO-based AWM and PEPG, with baseline scores taken from \citet{choi2026rethinkingdesignspacereinforcement}.

\subsection{Main Results}
\label{subsec:main_results}

\paragraph{RVM outperforms the trajectory-based methods}
\Cref{tab:vbench_t2v,tab:vbench_i2v,tab:t2i} compare RVM with trajectory-based fine-tuning across T2V, I2V, and the T2I OCR task. On Wan2.1-1.3B, RVM achieves $84.13$ VBench Overall, outperforming FlowGRPO ($75.91$), DanceGRPO ($79.46$), and the TaRoS variants ($79.84$--$80.56$). On SkyReels, RVM reaches $86.27$, exceeding FlowGRPO by $8.30$ points. On the SD3.5-M OCR task, RVM improves OCR accuracy from FlowGRPO's $0.92$ to $0.95$ and surpasses it on all five accompanying metrics. These results demonstrate effective reward fine-tuning across image and video generation through direct velocity regression, without transition-likelihood updates.

\paragraph{Leading T2V and T2I performance}
On Wan2.1-T2V-1.3B (\cref{tab:vbench_t2v}), RVM achieves the highest scores on six of eight reported metrics across all methods and pretrained baselines, including a VBench Overall score of $84.13$, surpassing all evaluated trajectory- and velocity-based baselines. It exceeds DiffusionNFT and RAM by $1.18$ and $2.18$ Overall points, respectively, while attaining the highest Quality and Dynamic Degree. On SkyReels (\cref{tab:vbench_i2v}), RVM also leads in Quality and Dynamic Degree, with Overall performance comparable to DiffusionNFT and RAM. On SD3.5-M OCR (\cref{tab:t2i}), RVM achieves the best or tied-best results on five of six metrics and outperforms AWM and DiffusionNFT on all six. These results establish direct velocity regression as a strong approach across video and image generation, without likelihood-ratio estimation. We examine reward and anchor choices in \cref{subsec:ablation}.

% \paragraph{Strong video performance and leading T2I results}
% RVM compares favorably with other velocity-based objectives across all three tasks (\cref{tab:vbench_t2v,tab:vbench_i2v,tab:t2i}). On Wan2.1, RVM achieves $84.13$ VBench Overall, exceeding DiffusionNFT and RAM by $1.18$ and $2.18$ points, respectively, while also attaining the highest Quality and Dynamic Degree. On SkyReels, all three methods improve substantially over the pretrained model, with comparable Overall scores of $86.16$--$86.61$; RVM additionally achieves the highest Quality and Dynamic Degree. RVM is particularly strong on the SD3.5-M T2I OCR task, achieving the best or tied-best scores on five of six reported metrics. It outperforms AWM and DiffusionNFT on all six metrics and matches or exceeds the ELBO-based PEPG baseline on five, without estimating policy ratios. These results demonstrate that direct velocity regression is effective across video and image generation. We examine the effects of reward and anchor choices in \cref{subsec:ablation}. 
% \av{You should sell your results in table 1 stronger. You outperform other velocity based approaches by a large margin and don't mention it here. Instead the next sentence focuses on the table that has only comparable results.}\jaemoo{Revised by: appealing strongest result (Wan 2.1) then discuss skyreels.}

\paragraph{Lower training costs than trajectory-based methods}
\Cref{fig:quality_cost} compares generation quality against total training cost on Wan2.1. RVM achieves a $4.7$-point higher VBench Overall score than DanceGRPO at $11.8\times$ lower total cost, using approximately $525$ GPU-hours over $90$ iterations. The DanceGRPO cost in \cref{fig:quality_cost} is estimated by extrapolating measured rollout, reward, and update timings to $220$ training iterations, using $512$ videos per iteration ($32$ prompts with group size $16$), yielding approximately $6{,}171$ GPU-hours. To explain these savings, \cref{fig:train_cost} breaks down the estimated per-iteration costs at a common batch size of $256$ videos. RVM requires $5.8$ GPU-hours, compared with $13.5$ for FlowGRPO and $14.5$ for DanceGRPO, corresponding to approximately $2.3\times$ and $2.5\times$ lower costs. Both sampling and updates contribute: our deterministic ODE rollouts without CFG require fewer model evaluations, while velocity-regression updates use fewer noise levels than the compared trajectory-based configurations. These estimates use measured component timings and the respective training configurations, comparing the costs and quality of the reported runs rather than time to a common quality threshold.
% \av{Note that the results are shown a bit selectively. Figure 4 doesn't have the RAM baseline and Figure 5 is missing RAM and DiffusioNFT baselines.}\jaemoo{Comparison is vs. trajectory mainly, so changed paragraph name and some content!}

\subsection{Further Analysis}\label{subsec:ablation}
\flushbottom

\begin{figure*}[t]
    \centering
    \captionsetup{skip=2pt}
    \begin{minipage}[t]{0.48\linewidth}
        \vspace{0pt}
        \centering
        \includegraphics[width=\linewidth]{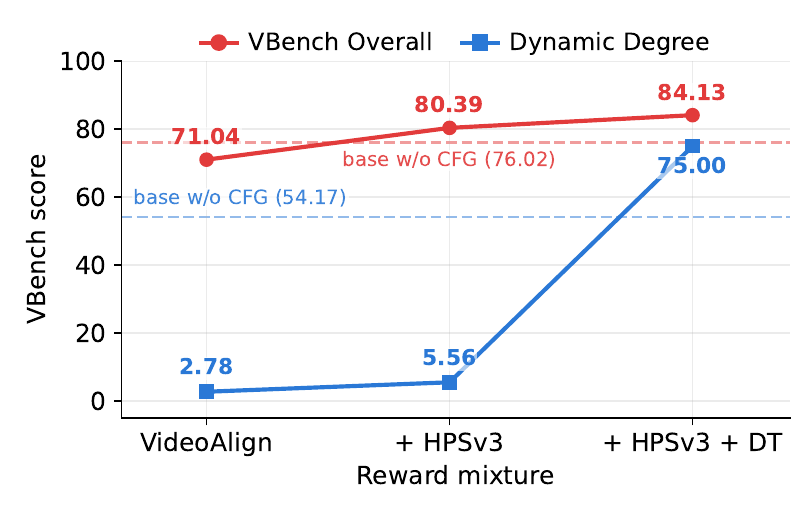}
        \caption{\textbf{Ablation of reward design on Wan2.1-T2V-1.3B.}
        {VBench Overall and dynamic degree for the three reward mixtures.}}
        \label{fig:reward_ablation}
    \end{minipage}\hfill
    \begin{minipage}[t]{0.48\linewidth}
        \vspace{0pt}
        \centering
        \includegraphics[width=\linewidth]{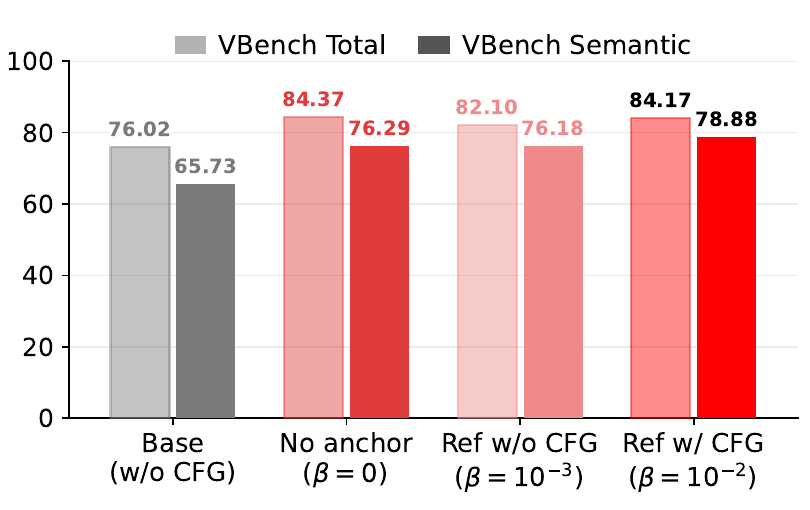}
        \caption{\textbf{Ablation of the anchor velocity on Wan2.1-T2V-1.3B.}}
        \label{fig:anchor_ablation}
    \end{minipage}
\end{figure*}

\paragraph{Reward design}
Reward design determines which generation behaviors are reinforced during fine-tuning, making it essential to capture the properties we ultimately want to improve. For video generation, this requires rewarding both visual quality and motion. Common choices include \textbf{(i) VideoAlign}~\citep{liu2025videoalign}, which provides rewards for text alignment (TA), motion quality (MQ), and visual quality (VQ). For \textbf{(ii) VideoAlign + HPSv3}, we incorporate human-preference feedback using HPSv3~\citep{ma2025hpsv3}. In our experiments, however, fine-tuning with these reward mixtures substantially reduces motion, even when a motion-quality reward is included. This motivates our \textbf{dynamic-tracking (DT) reward}, which uses RAFT optical flow~\citep{teed2020raft} to explicitly reward motion dynamics. To examine the contributions of preference and explicit motion feedback, we compare \textbf{(i)} and \textbf{(ii)} with \textbf{(iii) VideoAlign + HPSv3 + DT}. \Cref{fig:reward_ablation} reveals that improved aggregate scores can conceal a substantial loss of motion. VideoAlign alone yields a Dynamic Degree of just $2.78$, compared with $54.17$ for the pretrained model without CFG. Replacing VQ with HPSv3 raises VBench Overall from $71.04$ to $80.39$, yet leaves Dynamic Degree at only $5.56$. Stronger
preference feedback improves the overall score without recovering motion. Adding DT changes this outcome, raising Dynamic Degree to $75.00$ while further improving Overall to $84.13$, surpassing the pretrained baseline on both measures. Thus, explicit motion feedback enables simultaneous gains in motion and aggregate performance in this setting. We use the third mixture for video training; weights and model-specific reward definitions are provided in \cref{app:videoreward}.
% \av{Start this part with a few sentences that highlights why designing a good reward is important.}\jaemoo{Revised!}

\paragraph{Anchor velocity}
We compare training without an anchor with training anchored to the frozen base-model velocity evaluated either without CFG or with CFG where $\beta>0$. \Cref{fig:anchor_ablation} shows that the model trained with CFG-guided anchor achieves a Semantic score of \revred{$78.88$}, compared with $76.18$ for the model trained with the reference anchor without CFG and $76.29$ for the model trained without an anchor. The positive anchor term in \cref{eq:rvm} penalizes deviations from the guided reference field during training, while inference uses the learned velocity without CFG. The unanchored setting retains the highest VBench Total, reaching $84.37$ compared with \revred{$84.17$} for the CFG-guided anchor.

% \paragraph{\revblue{Few-step generation}} \jaemoo{I'll put this into appendix with better explanation.} \av{It's not clear why this experiment is done here? Are you reducing the number of sampling steps when rolling out during training or purely in inference? What would you like to say with this experiment? You could say that RL-finetuned models appear to have simpler trajectories that are less sensitive to discretization error.}
% \revblue{We compare RVM without CFG with the CFG-guided base model under the same $16$-, $5$-, and $4$-step inference schedules. \Cref{fig:step_ablation} shows that RVM's VBench Total decreases from $84.12$ at $16$ steps to $83.69$ and $82.29$ at $5$ and $4$ steps, respectively. The corresponding base-model scores decrease from $83.10$ to $73.62$ and $64.39$. RVM exhibits a smaller decrease in VBench Total under these schedules. \Cref{app:impl} specifies the evaluation schedules.
% }

\section{Conclusion}
We take velocity matching as the starting point for reward fine-tuning, yielding RVM and a common regression structure connecting RAM and DiffusionNFT. Across image and video benchmarks, RVM achieves performance on par with or better than the evaluated baselines, with lower estimated training costs than trajectory-based methods on Wan2.1-T2V-1.3B. Exploring better reward designs and anchor choices is an interesting direction for future work. Extending RVM to autoregressive video generation would also be a promising direction.
\par\endgroup
\flushbottom

% \section*{Acknowledgments}
% The authors are grateful for partial supports from NSF Grants ECCS-1942523, DMS-2206576, 2409016, 2450378, and AFOSR Grant FA9550-25-1-0169. JC is supported by National Research Foundation of Korea (NRF) grants (RS-2024-00342044). YZ and MT gratefully acknowledge the partial supports by NSF Grant DMS-2513699 (YZ \& MT), DOE Grants NA0004261 (MT), SC0026274 (YZ \& MT), Richard Duke Fellowship (YZ \& MT), and Simons Institute for the Theory of Computing at UC Berkeley (MT). WG acknowledges Georgia Tech ARC-ACO Fellowship for partial support.

\section*{AI Use Statement}
We used generative AI tools to improve the readability of the manuscript, including polishing sentences and correcting grammar. The authors reviewed all AI-assisted edits and take full responsibility for the content of this work.

\section*{Reproducibility Statement}
\Cref{sec:loss} provides the loss definitions, reward-weight conventions, and derivations connecting the objectives. \Cref{app:experiment_setup} describes optimization, sampling, loss normalization, training and evaluation data, and the anchor evaluation protocol; \cref{tab:hparams} lists the per-model hyperparameters and hardware. \Cref{app:t2i_impl} details the text-to-image OCR task, including its training rewards and evaluation settings, while \cref{app:videoreward} defines the video rewards and their mixture weights. \Cref{app:few_step} specifies the few-step inference schedules and evaluation procedure. Full per-dimension video benchmark results and additional qualitative comparisons are provided in \cref{app:vbench_full} and \cref{app:qualitative}, respectively.

% \section*{Impact Statement}
% This paper presents work aimed at advancing the field of generative models. It proposes a simple approach to reward fine-tuning of video diffusion models. For example, it may be abused to create various harmful and offensive content. We strongly caution the community against such use cases.

% \clearpage
\bibliographystyle{iclr2027_conference}
\bibliography{reference.bib}

%%%%%%%%%%%%%%%%%%%%%%%%%%%%%%%%%%%%%%%%%%%%%%%%%%%%%%%%%%%%%%%%%%%%%%%%%%%%%%%
%%%%%%%%%%%%%%%%%%%%%%%%%%%%%%%%%%%%%%%%%%%%%%%%%%%%%%%%%%%%%%%%%%%%%%%%%%%%%%%
% APPENDIX
%%%%%%%%%%%%%%%%%%%%%%%%%%%%%%%%%%%%%%%%%%%%%%%%%%%%%%%%%%%%%%%%%%%%%%%%%%%%%%%
%%%%%%%%%%%%%%%%%%%%%%%%%%%%%%%%%%%%%%%%%%%%%%%%%%%%%%%%%%%%%%%%%%%%%%%%%%%%%%%
\newpage
\appendix
% Appendix revisions are typeset in the original black text.
\renewcommand{\revblue}[1]{#1}
\color{black}
\hypersetup{linkcolor=black,citecolor=black,urlcolor=black}
\crefalias{section}{appendix}
\crefalias{subsection}{appendix}
\crefalias{subsubsection}{appendix}
\onecolumn

\section{Related Works}\label{app:related}
\paragraph{RL fine-tuning for language models}
RL has been extensively studied as a principled approach for improving language-model capabilities and alignment \citep{ouyang2022training, guo2025deepseek}. Popular methods include PPO \citep{schulman2017proximal}, GRPO \citep{shao2024deepseekmath}, and their variants \citep{liu2025understanding, yu2025dapo, ahmadian2024back}. Recent algorithmic developments focus on designing more stable policy objectives \citep{zhao2025geometric, zheng2025group, gao2025soft, chen2025minimax, team2025kimi} and understanding their estimation and optimization behavior \citep{zhang2025design, zheng2025stabilizing, liu-li-2025-rl-collapse}. QUATRO \citep{lee2026quatro} is closely related to our policy-gradient formulation, as it studies proximal sequence-level policy optimization for likelihood-based LLMs. However, directly transferring these methods to diffusion and flow models is nontrivial because diffusion-model likelihoods and likelihood ratios are intractable.

\paragraph{Trajectory-based RL for diffusion and flow models}
RL has also been widely adopted to fine-tune diffusion and flow models toward human preferences or task-specific rewards \citep{fan2023dpok, black2023training, domingo2024adjoint}. \revred{One line of work follows a trajectory-based view}, where the denoising or flow trajectory is treated as the policy trajectory. FlowGRPO \citep{liu2025flow} adapts GRPO to diffusion models \revred{using trajectory-based likelihood estimation for text-to-image reward fine-tuning}. \revred{Subsequent works modify the trajectory estimators, samplers, or rollout structures} \citep{he2025tempflow, wang2025grpo, xue2025dancegrpo, wang2025pref, ye2025data, xue2025advantage}. Other examples include MixGRPO \citep{li2026mixgrpounlockingflowbasedgrpo}, BranchGRPO \citep{li2025branchgrpostableefficientgrpo}, and LeapAlign \citep{liang2026leapalignposttrainingflowmatching}, which modify the rollout or trajectory construction. Reward Score Matching \citep{lee2026rewardscorematchingunifying} unifies many of these trajectory-based methods as score-matching regressions toward a value-guided target, and can be viewed as a trajectory-level counterpart of the ELBO-based approaches \cref{eq:epg2}. \revred{For FlowGRPO and DanceGRPO, we quantify rollout and gradient-update costs under the Wan2.1 configurations evaluated in \cref{subsec:main_results}.}

\paragraph{Simplified and trajectory-independent diffusion fine-tuning}
A separate line of work seeks simpler alternatives that avoid explicit trajectory-level likelihood estimation. DiffusionNFT \citep{zheng2025diffusionnft} adapts negative fine-tuning \citep{chen2025bridging} to diffusion and flow models by directly modifying velocity-regression losses. AWM \citep{xue2025advantage} and RAM \citep{bergmeister2026reinforceadjointmatchingscaling} similarly replace trajectory-based policy-gradient estimation with simpler velocity- or score-based objectives. Relatedly, DPPO \citep{ren2025diffusion} applies PPO-style updates to diffusion policies in continuous-control and robotic-manipulation settings, while Q-score matching \citep{psenka2024learning} learns diffusion policies from reward signals by matching score functions with Q-function gradients. \citet{choi2026rethinkingdesignspacereinforcement} \revred{show} that, for diffusion fine-tuning, simple objectives can perform competitively when the likelihood or velocity-based training loss is chosen carefully. Our work follows this direction, focusing on trajectory-independent velocity-based losses and analyzing when objective engineering is actually necessary. Concurrently, \citet{xu2026designingreinforcementlearning} present a unified path-space view in which reverse-trajectory and forward-matching methods arise from a common importance-sampling-based policy-gradient principle, attributing empirical gaps between method families to variance reduction rather than differing RL principles. Their unification operates on trajectory-space likelihood ratios, whereas RVM unifies velocity-based methods directly in the model's native velocity representation, avoiding likelihood estimation altogether.
Concurrent work by \citet{anonymous2026flowtrpo} develops variants of velocity-matching RL objectives. Complementing our focus on objective unification and reward design, this work explores combinations with search algorithms and flow maps to scale to challenging rewards.

\paragraph{RL fine-tuning for video diffusion models}
Recent works have begun to apply RL-style fine-tuning to video diffusion and autoregressive video generators. DanceGRPO \citep{xue2025dancegrpo} and FlowGRPO-based video extensions apply GRPO-style objectives to video generation, while KVPO \citep{zhang2026kvpoodenativegrpoautoregressive}, RAVEN \citep{lu2026ravenrealtimeautoregressivevideo}, TAGRPO \citep{wang2026tagrpoboostinggrpoimagetovideo}, SAGE-GRPO \citep{zheng2026manifoldawareexplorationreinforcementlearning}, AR-CoPO \citep{he2026arcopoalignautoregressivevideo}, V-GRPO \citep{tang2026vgrpoonlinereinforcementlearning}, and recent systematic video fine-tuning frameworks \citep{xue2026systematicposttrainframeworkvideo} explore different objectives, rollout strategies, and video-specific reward designs. Other works focus on reward construction, including TaRoS \citep{li2026rethinkingrewardsignalsvideo}, Adv-GRPO \citep{mao2025imagerewardreinforcementlearning}, Diffusion-DRF \citep{wang2026diffusiondrffreerichdifferentiable}, AlphaGRPO \citep{huang2026alphagrpounlockingselfreflectivemultimodal}, and PromptRL \citep{wang2026promptrlpromptmattersrl}. \revred{With the evaluated reward mixtures, RVM exceeds FlowGRPO in VBench Overall on Wan2.1 and VBench-I2V Overall on SkyReels, while requiring lower estimated training cost than FlowGRPO and DanceGRPO on Wan2.1 (\cref{subsec:main_results}).}

\section[Discussion of Loss Functionals]{Discussion of Loss Functionals}\label{sec:loss}

We collect here the loss functionals discussed in \cref{sec:bg-objectives} in a common notation and describe the reward weight and the per-loss implementation choices used in our runs. Throughout, $\bm{x}^{1:G}_0$ is a group of $G$ samples drawn for a shared prompt from each method's sampling policy, a general sampling $\piold$ for RVM, the current policy $\pitheta$ for RAM, and its exponential moving average $\piold$ for DiffusionNFT, while $\bm{x}_t^i\sim p_{t\mid 0}(\cdot\mid\bm{x}_0^i)$ is a single noisy state with velocity target $\vtgti$, and $r^i$ is the reward weight of the $i$-th sample.

\paragraph{Reward weight}
\revred{For video training, we center rewards within each prompt group and normalize by the standard deviation across the full batch. The implemented standardized reward is}
\begin{align*}
r^i_{\mathrm{grpo}} = \frac{R(\bm{x}_0^i) - \operatorname{mean}\big(R(\bm{x}_0^1),\dots,R(\bm{x}_0^G)\big)}{\revred{\operatorname{std}_{\mathrm{batch}}(R)}},
\end{align*}
\revred{with reward weights $r^i_{\mathrm{rvm}}=r^i_{\mathrm{ram}}=r^i_{\mathrm{grpo}}$ for the video runs. OCR instead uses prompt-group mean subtraction without standard-deviation normalization (\cref{app:t2i_impl}). These settings are task- and method-specific, rather than shared by all baselines. DiffusionNFT maps the unscaled standardized reward to a convex mixing weight $r^i_{\mathrm{nft}}\in[0,1]$, as detailed below.}

\subsection{Velocity-based Loss Functionals}
\label{subapp:vel_loss}

\paragraph{RVM}
\revred{The objective in \cref{eq:rvm}} is the velocity-regression form we optimize, a reward-weighted velocity regression $r^i_{\mathrm{rvm}}\,\|\vtheta(\bm{x}_t^i,t)-\vtgti\|_2^2$ where the reward acts as a signed coefficient on the endpoint-defined velocity target. Positive rewards encourage the model to match the velocity direction associated with the generated sample, whereas negative rewards discourage that direction. It uses the \revred{task-specific reward weight described above} and the ELBO weighting $w(t)$ of \cref{eq:elbo} (\cref{subapp:pg_loss}), and forms no likelihood ratio and simulates no trajectory.

\paragraph{RAM}
RAM \cref{eq:ram} dispenses with the likelihood ratio entirely and instead shapes the velocity-regression target. In our runs the reward weight $r^i_{\mathrm{ram}}$ in \cref{eq:ram} is the prompt-group-centered, batch-normalized reward above; it multiplies the residual $\vtgti-\vtheta(\bm{x}_t^i,t)$ between the bridge target and the current velocity, nudging $\vtheta$ away from the frozen reference $\vref$ in proportion to $r^i_{\mathrm{ram}}$. The original formulation instead scales this residual by the raw scalar reward, with no group baseline \citep{bergmeister2026reinforceadjointmatchingscaling}. In either case the reward weight acts as a multiplicative scale on the regression target rather than as a weight on a log-probability, and the reference enters as an additive anchor inside the stop-gradient target instead of as a separate $\beta\,\|\vtheta-\vref\|_2^2$ regularizer.

\paragraph{DiffusionNFT}
DiffusionNFT \cref{eq:nft} also avoids log-probabilities, contrasting the two implicit velocity fields $v^{\pm}_\theta = (1\mp\bar\beta)\bm{v}_{\mathrm{anc}} \pm \bar\beta\vtheta$, where $\bm{v}_{\mathrm{anc}}$ is an EMA-updated velocity (an exponential moving average of the current policy) and $\bar\beta$ is the guidance strength; we use $\bar\beta = 0.1$. Its reward weight is a \emph{squashed} reward $r^i_{\mathrm{nft}}\in[0,1]$ that acts as a convex mixing weight between a positive branch, pulled toward the bridge target $\vtgti$, and a negative branch, pushed away from it. Because the weights $r^i_{\mathrm{nft}}$ and $1-r^i_{\mathrm{nft}}$ must form a convex combination, the standardized reward cannot be used directly; we map it into $[0,1]$ by clipping and rescaling. Concretely, $r^i_{\mathrm{nft}}$ is obtained from \revred{the unscaled, batch-normalized reward $r^i_{\mathrm{grpo}}$ defined above} by
\begin{align*}
r^i_{\mathrm{nft}} = \operatorname{clip}\!\left( \tfrac{1}{2} + \tfrac{1}{10}\,\operatorname{clip}(r^i_{\mathrm{grpo}}, -5, 5),\; 0,\; 1 \right),
\end{align*}
so that a group-average sample ($r^i_{\mathrm{grpo}} = 0$) maps to $\tfrac12$, while $r^i_{\mathrm{grpo}} = +5$ and $r^i_{\mathrm{grpo}} = -5$ map to $1$ and $0$, respectively (the inner clip range $\pm 5$ is the clipping constant used \revred{here}). High-reward samples ($r^i_{\mathrm{nft}}\to1$) \revred{place more weight on the positive regression branch}, while low-reward samples ($r^i_{\mathrm{nft}}\to0$) \revred{place more weight on the negative regression branch}. \revred{For DiffusionNFT-specific implementation details, including EMA decay and update scheduling, loss normalization, timestep sampling, and training budget, see \citet{zheng2025diffusionnft}.}

\subsection{Connections between the Loss Functionals}
\label{subapp:connections}

\paragraph{Proof of the unification}
RAM and DiffusionNFT are motivated by different principles: RAM is derived from adjoint matching~\citep{bergmeister2026reinforceadjointmatchingscaling}, whereas DiffusionNFT is derived from contrastive flow matching~\citep{zheng2025diffusionnft}. Nevertheless, their velocity-field updates share the same anchored regression form as \cref{eq:rvm}. Specifically, RAM corresponds to choosing the frozen reference velocity as the anchor with unit anchor strength, while DiffusionNFT corresponds to choosing the EMA velocity as the anchor with reward-dependent anchor strength $\bar\beta-r^i_\mathrm{rvm}$.

\begin{proof}[Proof of \cref{thm:unify}]
Omit the shared argument $(\bm{x}_t^i,t)$ on every velocity and recall the RVM gradient $\nabla_{\vtheta} \cL_{\mathrm{rvm}} = r^i(\vtheta-\vtgti) + \beta(\vtheta-\bm{v}_{\mathrm{anc}})$.

\noindent\emph{(i) RAM.} The target of \cref{eq:ram} is detached, so the squared-error gradient is current minus target,
\begin{align}
\nabla_{\vtheta} \cL_{\mathrm{ram}} = \vtheta - \operatorname{sg}\!\big[\vref + r^i_{\mathrm{ram}}(\vtgti-\vtheta)\big] = r^i_{\mathrm{ram}}\,(\vtheta-\vtgti) + (\vtheta-\vref), \label{eq:ram-unify}
\end{align}
which is $\nabla_{\vtheta} \cL_{\mathrm{rvm}}$ at $\bm{v}_{\mathrm{anc}}=\vref$, $\beta=1$, $r^i=r^i_{\mathrm{ram}}$. The two gradients agree at every $\vtheta$.

\noindent\emph{(ii) DiffusionNFT.} The branches $\bm{v}^{\pm}_\theta = \bm{v}_{\mathrm{anc}} \pm \bar\beta(\vtheta-\bm{v}_{\mathrm{anc}})$ are affine in $\vtheta$ with slope $\pm\bar\beta$, so differentiating \cref{eq:nft} and substituting the branches gives, with $r^i_{\mathrm{rvm}} = 2r^i_{\mathrm{nft}}-1$,
\begin{align}
\nabla_{\vtheta} \cL_{\mathrm{nft}}
= r^i_{\mathrm{nft}}\,(\bm{v}^{+}_\theta-\vtgti) - (1-r^i_{\mathrm{nft}})\,(\bm{v}^{-}_\theta-\vtgti)
= \bar\beta\,(\vtheta-\bm{v}_{\mathrm{anc}}) - r^i_{\mathrm{rvm}}\,(\vtgti-\bm{v}_{\mathrm{anc}}). \label{eq:nft-grad}
\end{align}
Inserting $\vtgti-\bm{v}_{\mathrm{anc}} = (\vtgti-\vtheta)+(\vtheta-\bm{v}_{\mathrm{anc}})$ rearranges the right-hand side into $r^i_{\mathrm{rvm}}\,(\vtheta-\vtgti) + (\bar\beta - r^i_{\mathrm{rvm}})\,(\vtheta-\bm{v}_{\mathrm{anc}})$, which is $\nabla_{\vtheta} \cL_{\mathrm{rvm}}$ at the EMA anchor $\bm{v}_{\mathrm{anc}}$ with reward $r^i=r^i_{\mathrm{rvm}}$ and anchor strength $\beta(r^i_{\mathrm{rvm}}) = \bar\beta - r^i_{\mathrm{rvm}}$. The two gradients agree at every $\vtheta$ and, in this case, the scalar losses also differ only by $\theta$-independent constants, recovering \cref{eq:nft-rvm-main}.
\end{proof} 

\paragraph{A shared velocity-regression form}
Although the three objectives are written differently, their velocity-field gradients can all be represented by a common anchored velocity-regression surrogate at the noisy state $\bm{x}_t^i$:
\begin{align}
    \mathbb{E}_{\bm{x}^{1:G}_0\sim\piold}
\left[ \frac{c(r^i)}{2} \,\big\|\, \big(\vtheta(\bm{x}^i_t , t) - \bm{v}_\mathrm{anc}(\bm{x}^i_t , t)\big) - A(r^i)\,\big(\vtgti - \bm{v}_\mathrm{anc}(\bm{x}^i_t , t)\big)  \big\|_2^2 \right] ,
    \label{eq:shared-form}
\end{align}
where $c:\mathbb{R}\to\mathbb{R}$ is a reward-dependent scale and $A:\mathbb{R}\to\mathbb{R}$ a reward-dependent reach. Recall that $\vtgti := \bm{\epsilon}^i - \bm{x}^i_0$. The term $\vtheta-\bm{v}_\mathrm{anc}$ measures how far the current velocity has moved from the anchor, while $\vtgti-\bm{v}_\mathrm{anc}$ is the update direction from the anchor toward the flow-matching target $\vtgti$; the reach $A(r^i)$ sets how far the regression target moves along this direction and the scale $c(r^i)$ its overall weight. Each method is thus fixed by its anchor $\bm{v}_\mathrm{anc}$ and reward-dependent pair $(c,A)$, read off in \cref{tab:shared-form}. For the RVM row, this completed-square parameterization assumes $r_\mathrm{rvm}+\beta\ne0$; the original two-term objective \cref{eq:rvm} remains the definition in degenerate cases. Notably, the reach of DiffusionNFT divides the signed reward by the guidance strength, so its regression target lands beyond $\vtgti$ whenever $2r^i_{\mathrm{nft}}-1>\bar\beta$.

\begin{table}[h]
\centering
\footnotesize
\caption{\textbf{Velocity-based objectives through the shared regression form} \cref{eq:shared-form}.
With $\piold := \pitheta$, $\beta = 1$, and $\bm{v}_\mathrm{anc} := \bm{v}_\mathrm{ref}$, RAM's gradient matches that of RVM. With $\bm{v}_\mathrm{anc} :=$ EMA velocity, $r_\mathrm{rvm} = 2r_\mathrm{nft} - 1$, and $\beta = \bar\beta - r_\mathrm{rvm}$, DiffusionNFT equals RVM up to $\theta$-independent constants.}
\label{tab:shared-form}
\begin{tabular}{lcccc}
\toprule
Method & Sampling policy & Anchor $\bm{v}_\mathrm{anc}$ & Scale $c(r)$ & Reach $A(r)$ \\
\midrule
RAM & $\pitheta$ & $\vref$ & {$r_\mathrm{ram}+1$} & $\tfrac{{r^i_\mathrm{ram}}}{{r_\mathrm{ram}+1}}$ \\
DiffusionNFT & $\piold$ & EMA velocity & {$\bar\beta$} & {$(2r_{\mathrm{nft}}-1)/\bar\beta$} \\
\revred{RVM} & $\piold$ & any & ${\beta} {\big({r_\mathrm{rvm}}/{\beta}+1\big)}$ & $ \tfrac{{r_\mathrm{rvm}/{\beta}}}{{r_\mathrm{rvm}/\beta+1}}$ \\
\bottomrule
\end{tabular}
\end{table}

\subsection{Policy-gradient-based Loss Functionals}
\label{subapp:pg_loss}

\revblue{
Policy-gradient methods incorporate reward feedback through likelihood ratios. Their connection to velocity regression depends on how these ratios are estimated and differentiated.

\paragraph{Trajectory-based likelihood ratios}
FlowGRPO~\citep{liu2025flow} applies the clipped GRPO objective to the Gaussian transition probabilities of an SDE rollout. Each transition ratio is clipped separately and weighted by the group-relative reward $r^i_{\mathrm{grpo}}$. The resulting regression target depends on the sampled intermediate state, as derived in \cref{app:traj_vel}.

\paragraph{ELBO-based likelihood ratios}
EPG and PEPG~\citep{choi2026rethinkingdesignspacereinforcement} instead estimate endpoint likelihood ratios using independently forward-noised samples. Applying \cref{eq:elbo} at a sampled state gives
\begin{equation*}
\hat{\rho}_\theta(\bm{x}_t,t)
=\exp\!\left(w(t)\left[\|\vold(\bm{x}_t,t)-\vtgt\|_2^2-\|\vtheta(\bm{x}_t,t)-\vtgt\|_2^2\right]\right).
\end{equation*}
This estimate supplies the detached importance factor $\bar\rho_\theta$ in \cref{eq:epg2}. We consider uniform weighting $w(t)=1$ and adaptive weighting with a detached, sample-dependent normalizer. The adaptive choice depends on the regression residual as well as time; it is held fixed during differentiation. Reference regularization is approximated by $\|\vtheta(\bm{x}_t,t)-\vref(\bm{x}_t,t)\|_2^2$, following \citet{zheng2025diffusionnft,xue2025advantage}.

\paragraph{PEPG adjusts the reward weight}
EPG uses the group-relative reward $r^i_{\mathrm{epg}}=r^i_{\mathrm{grpo}}$ together with a detached importance ratio. PEPG additionally subtracts the detached log-ratio from this reward weight. Writing $\rhotheta(\bm{x}_0)=\pitheta(\bm{x}_0)/\piold(\bm{x}_0)$ and $\bar\rho_\theta=\operatorname{sg}[\rhotheta]$, its objective is
\begin{align}
\mathcal{L}_{\mathrm{pepg}}
=\mathbb{E}_{\bm{x}^{1:G}_0\sim\piold}
\big[
-\bar\rho_\theta(\bm{x}_0^i)
\big(r^i_{\mathrm{grpo}}-\log\bar\rho_\theta(\bm{x}_0^i)\big)
\log\rhotheta(\bm{x}_0^i)
+\beta\operatorname{kl}(\bm{x}_0^i)
\big], \tag{PEPG}\label{eq:pepg}
\end{align}
where $\operatorname{kl}(\bm{x}_0^i)$ is the per-sample reference regularization term, approximated by the velocity penalty above. Samples and reward weights are detached, so the reward term differentiates only through $\log\rhotheta$; the reference penalty contributes its own gradient. The additional detached log-ratio reduces the reward weight for samples whose likelihood has increased relative to $\piold$.

\paragraph{Connection to RVM}
Substituting $\log\hat{\rho}_\theta$ into these objectives gives a reward-weighted velocity-regression update because the sampling-policy error is fixed. With $w(t)=1$, the EPG surrogate is \cref{eq:epg2}. It matches the RVM update, up to overall scaling, when the detached ratio is included in $r^i_{\mathrm{rvm}}$, the anchor is $\bm{v}_{\mathrm{anc}}=\vref$, and the anchor coefficients agree. In the on-policy case, the ratio is one, giving the AWM specialization under matched reward, timestep weighting, and stop-gradient conventions. PEPG retains the additional $-\log\bar\rho_\theta$ correction in its regression weight. These connections concern the implemented surrogate gradients; they do not equate the ELBO estimates with exact endpoint likelihoods.
}

\subsection{\texorpdfstring{Trajectory-based Objectives as Velocity Regression}{Trajectory-based Objectives as Velocity Regression}}\label{app:traj_vel}

\revblue{Trajectory-based policy-gradient updates admit a velocity-regression representation whose target depends on the sampled transition. In this subsection, $\bm{x}^i$ denotes a complete SDE rollout. Fix a denoising time $t$ and let expectations average over grouped rollouts $\bm{x}^{1:G}\sim\piold$ and a uniform group index $i$, with samples and rewards detached during differentiation. Omitting clipping, the per-step GRPO surrogate is}
\begin{align}
\mathcal{L}_{\mathrm{traj}} = \revblue{\mathbb{E}_{\bm{x}^{1:G}\sim\piold}}\Big[-r^i_{\mathrm{grpo}}\,\frac{\pitheta(\bm{x}^i_{t-\Delta t}\mid\bm{x}^i_t)}{\piold(\bm{x}^i_{t-\Delta t}\mid\bm{x}^i_t)}\Big],
\end{align}
whose gradient, by the identity $\nabla_\theta\pitheta = \pitheta\nabla_\theta\log\pitheta$, coincides with that of
\begin{align}
\revblue{\mathbb{E}_{\bm{x}^{1:G}\sim\piold}}\Big[-r^i_{\mathrm{grpo}}\,\operatorname{sg}\!\Big[\tfrac{\pitheta(\bm{x}^i_{t-\Delta t}\mid\bm{x}^i_t)}{\piold(\bm{x}^i_{t-\Delta t}\mid\bm{x}^i_t)}\Big]\log\pitheta(\bm{x}^i_{t-\Delta t}\mid\bm{x}^i_t)\Big].
\label{eq:traj-logtrick}
\end{align}
The transition log-likelihood is available in closed form. Under the SDE samplers of FlowGRPO and DanceGRPO, one denoising step is the Gaussian
\begin{align}
\bm{x}_{t-\Delta t} = \bm{\mu}_\theta(\bm{x}_t, t) + \sigma_t\sqrt{\Delta t}\,\bm{\epsilon},\qquad
\bm{\mu}_\theta = \bm{x}_t - \Delta t\Big[\Big(1+\tfrac{\sigma_t^2(1-t)}{2t}\Big)\vtheta(\bm{x}_t,t) + \tfrac{\sigma_t^2}{2t}\bm{x}_t\Big],
\end{align}
where $\sigma_t$ is the noise scale of the sampler; FlowGRPO uses $\sigma_t = a\sqrt{t/(1-t)}$ with stochasticity level $a$, and DanceGRPO a constant. Completing the square in $\vtheta$ gives
\begin{align}
-\log\pitheta(\bm{x}_{t-\Delta t}\mid\bm{x}_t)
= w_{\mathrm{traj}}(t)\,\big\|\vtheta(\bm{x}_t, t) - \bm{v}_{\mathrm{traj}}\big\|_2^2 + C,
\nonumber\\
w_{\mathrm{traj}}(t) = \frac{\big(2t+\sigma_t^2(1-t)\big)^2\,\Delta t}{8\,\sigma_t^2\,t^2},
\label{eq:traj-nll}
\end{align}
with $C$ collecting all $\theta$-independent terms and the induced regression target
\begin{align}
\bm{v}_{\mathrm{traj}}
= \frac{2t}{2t+\sigma_t^2(1-t)}\cdot\frac{\bm{x}_t-\bm{x}_{t-\Delta t}}{\Delta t}
\;-\;\frac{\sigma_t^2}{2t+\sigma_t^2(1-t)}\,\bm{x}_t.
\label{eq:traj-target}
\end{align}
\revblue{Substituting \cref{eq:traj-nll} into \cref{eq:traj-logtrick} yields reward-weighted velocity regression toward $\bm{v}_{\mathrm{traj}}$. This target is an affine transformation of the sampled direction $(\bm{x}_t-\bm{x}_{t-\Delta t})/\Delta t$, reducing to that direction as $\sigma_t\to0$. For positive reward weights, the update shifts the transition mean toward the sampled next state. RVM instead uses the generated sample to define $\vtgt$, whose negative points toward $\bm{x}_0$ in the denoising direction of decreasing time (\cref{fig:traj_vs_vel}).}

\paragraph{\revblue{The sampled transition adds target noise}}
\revblue{Substitution of the sampling-policy mean into \cref{eq:traj-target} gives conditional mean $\vold(\bm{x}_t,t)$ and the centered second moment
\begin{equation}
\mathbb{E}_{\bm{x}_{t-\Delta t}\sim\piold(\cdot\mid\bm{x}_t)}
\!\left[(\bm{v}_{\mathrm{traj}}-\vold)(\bm{v}_{\mathrm{traj}}-\vold)^{\mathsf T}\,\middle|\,\bm{x}_t,t\right]
=\frac{4t^2\sigma_t^2}{\big(2t+\sigma_t^2(1-t)\big)^2\Delta t}\,\bm{I}.
\label{eq:traj-target-covariance}
\end{equation}
Here $\vold$ is evaluated at $(\bm{x}_t,t)$ and the expectation is over the sampled reverse transition. The identity follows from its Gaussian noise variance $\sigma_t^2\Delta t\,\bm{I}$ and the coefficient of $\bm{x}_{t-\Delta t}$ in \cref{eq:traj-target}, before conditioning on the final reward. RVM constructs $\vtgt=\bm{\epsilon}-\bm{x}_0=(\bm{x}_t-\bm{x}_0)/t$ from generated samples and independent forward noise, eliminating dependence on the sampled reverse transition. Comparisons of complete reward-weighted gradient variances additionally require consistent sampling and weighting assumptions. \citet{xue2025advantage} derives a target-noise ordering under a common forward-process law.}

\subsection{\texorpdfstring{\revblue{Proximal Updates with Negative Rewards}}{Proximal Updates with Negative Rewards}}\label{app:proximal_rvm}

\revblue{
Anchoring can make the population RVM objective well posed even when reward weights are negative. Fix the prompt and sampling policy $\piold$, and take a constant $\beta>0$. Here, $\mathbb{E}_{\piold}$ averages over independent samples $\bm{x}_0^{1:G}\sim\piold$, a uniform group index $i$, uniform $t\in(0,1)$, and independent Gaussian forward noise. Samples, reward weights, and the anchor are detached. We minimize \cref{eq:rvm} over square-integrable velocity fields, assuming bounded weights and finite second moments for $\vtgti$ and $\bm{v}_{\mathrm{anc}}$.

Conditioning on $(\bm{x}_t^i,t)$ gives a quadratic whose curvature is $\beta+\mathbb{E}_{\piold}[r^i_{\mathrm{rvm}}\mid\bm{x}_t^i,t]$. If this quantity has a strictly positive uniform lower bound, the unique minimizer, up to equality almost everywhere, is
\begin{equation}
\vtheta
=\frac{\beta\bm{v}_{\mathrm{anc}}+\mathbb{E}_{\piold}[r^i_{\mathrm{rvm}}\vtgti\mid\bm{x}_t^i,t]}{\beta+\mathbb{E}_{\piold}[r^i_{\mathrm{rvm}}\mid\bm{x}_t^i,t]},
\label{eq:rvm-proximal-solution}
\end{equation}
where velocity arguments $(\bm{x}_t^i,t)$ are omitted. The first-order condition gives this expression; the denominator bound and conditional Jensen's inequality ensure square integrability. Thus, individual weights may be negative as long as the total conditional curvature remains positive.

\revred{When the anchor is the detached pre-update field, the objective has the form of a proximal subproblem with parameter $\beta^{-1}$. If the unanchored reward-weighted objective is convex in the velocity field, this is a standard convex proximal step~\citep{parikh2014proximal}. Otherwise, it defines a well-posed generalized proximal step under the strictly positive uniform lower bound on total conditional curvature stated above.} A frozen reference anchor instead regularizes toward that reference. This result concerns one population update and does not establish convergence of repeated updates or neural-network optimization.
}

\subsection{\texorpdfstring{\revblue{Unanchored Matching with Exponential Reward Weights}}{Unanchored Matching with Exponential Reward Weights}}\label{app:reward_reweighting}

\revblue{
Unanchored RVM fits a KL-regularized reward optimum when its weights are exponential in the raw reward. For bounded $r(\bm{x}_0)$ and reward strength $\lambda>0$, this optimum is
\begin{equation}
\pi_\lambda^\star
=\argmin_{\pi\ll\piold}\left\{\operatorname{KL}(\pi\|\piold)-\lambda\mathbb{E}_{\pi}[r(\bm{x}_0)]\right\}
=\frac{\piold(\bm{x}_0)e^{\lambda r(\bm{x}_0)}}{\mathbb{E}_{\piold}[e^{\lambda r(\bm{x}_0)}]}.
\label{eq:rvm-optimal-distribution}
\end{equation}
Substitution rewrites the minimized objective as $\operatorname{KL}(\pi\|\pi_\lambda^\star)$ plus a constant, establishing uniqueness. Comparison with $\piold$ also gives $\mathbb{E}_{\pi_\lambda^\star}[r(\bm{x}_0)]\geq\mathbb{E}_{\piold}[r(\bm{x}_0)]$.

Setting $r^i_{\mathrm{rvm}}=e^{\lambda r(\bm{x}_0^i)}$ and $\beta=0$ changes the sampling measure in \cref{eq:rvm} to give
\begin{equation}
\cL_{\mathrm{rvm}}
=\frac{\mathbb{E}_{\piold}[e^{\lambda r(\bm{x}_0)}]}{2}\,
\mathbb{E}_{\pi_\lambda^\star}[\|\vtheta-\vtgti\|_2^2].
\label{eq:rvm-kl-target-regression}
\end{equation}
Expectations use the forward-noising law of \cref{app:proximal_rvm}, with endpoints drawn from the indicated distribution. With finite target second moments, the population minimizer is
\begin{equation}
\bm{v}_\lambda^\star(\bm{x}_t^i,t)
=\mathbb{E}_{\pi_\lambda^\star}[\vtgti\mid\bm{x}_t^i,t]
=\frac{\mathbb{E}_{\piold}[e^{\lambda r(\bm{x}_0^i)}\vtgti\mid\bm{x}_t^i,t]}{\mathbb{E}_{\piold}[e^{\lambda r(\bm{x}_0^i)}\mid\bm{x}_t^i,t]}.
\label{eq:rvm-optimal-velocity}
\end{equation}
If the model realizes this field and its ODE admits unique transport along the interpolation, exact integration from Gaussian noise generates $\pi_\lambda^\star$~\citep{lipman2023flow}. The two objectives then share an optimal generated-sample distribution; their losses and parameter gradients need not coincide. This interpretation requires exponential sample-wise weights and does not apply directly to the group-relative weights used in our video experiments.
}

\section{Implementation Details}\label{app:impl}

\subsection{Experiment Setup}\label{app:experiment_setup}
\paragraph{Default Settings} \revred{The following defaults describe the video RVM runs; OCR settings are given in \cref{app:t2i_impl}.} We fine-tune each base model with LoRA~\citep{hu2022lora}, keeping the pretrained backbone frozen. Fine-tuning is classifier-free-guidance-free. \revred{For video training, each generated clip contains $53$ frames at $15$ frames per second (fps), corresponding to approximately $3.5$ seconds.} Samples are drawn with a deterministic ODE sampler \citep{zhang2023fast, zhang2023gddim}. Among various ODE sampler, we use a deterministic $16$-step DPM-Solver-2 flow ODE~\citep{lu2022dpmsolver} at guidance scale $1.0$, and the loss is computed as reward-weighted velocity regression at a single noised state per generation, with no policy-ratio clipping. We optimize with AdamW ($\beta_1{=}0.9$, $\beta_2{=}0.999$, weight decay $10^{-4}$, $\epsilon{=}10^{-8}$) at learning rate $5\times10^{-5}$ in bf16, with gradient clipping at norm $1.0$ and a fixed seed. Each iteration draws a group of $K{=}8$ samples for each of $N$ prompts. The default RVM run uses no anchor ($\beta{=}0$), so each update is the reward-weighted velocity regression alone; \revblue{the reference-anchor comparison in \cref{subsec:ablation} uses the frozen base fields described in \cref{app:anchor_protocol}.} Per-model settings are listed in \cref{tab:hparams}.

\paragraph{Loss Normalization}
Following \citet{choi2026rethinkingdesignspacereinforcement} and \citet{zheng2025diffusionnft}, we normalize the matching loss by the scale of its reconstruction error. For each noised sample $\bm{x}_t^i$, the velocity prediction gives a clean-sample estimate $\hat{\bm{x}}_0^{i,\theta}=\bm{x}_t^i-t\vtheta(\bm{x}_t^i,t)$. Since $\hat{\bm{x}}_0^{i,\theta}-\bm{x}_0^i=-t(\vtheta(\bm{x}_t^i,t)-\vtgti)$, we use the normalized quadratic
\begin{equation}
\ell_{\mathrm{norm}}^i
=\frac{\|\hat{\bm{x}}_0^{i,\theta}-\bm{x}_0^i\|_2^2/d}
{\max\!\left(\operatorname{sg}\!\left[\|\hat{\bm{x}}_0^{i,\theta}-\bm{x}_0^i\|_1/d\right],\,10^{-5}\right)}
=\frac{t^2\|\vtheta(\bm{x}_t^i,t)-\vtgti\|_2^2/d}
{\max\!\left(\operatorname{sg}\!\left[t\|\vtheta(\bm{x}_t^i,t)-\vtgti\|_1/d\right],\,10^{-5}\right)},
\label{eq:loss-normalization}
\end{equation}
where $d$ is the number of latent elements in one sample. The numerator is the mean squared reconstruction error, and the denominator is its mean absolute error over all non-batch dimensions, including the temporal dimension for video. We recompute the denominator for each sample and timestep but detach it during backpropagation, so gradients act only on the quadratic numerator. The lower bound $10^{-5}$ avoids division by small errors. This gives a sample-dependent weighting rather than a function of $t$ alone. In RVM, the normalized matching term is weighted by $r^i_{\mathrm{rvm}}$ and averaged over samples; the separate anchor penalty retains its mean squared velocity difference without this normalization.

\paragraph{Training and evaluation data}
Both video models are trained on publicly available prompt sets. For the default Wan2.1-T2V-1.3B runs, we use the prompt corpus released by DanceGRPO~\citep{xue2025dancegrpo}, consisting of $48{,}998$ training prompts, which we use as released without further preprocessing. For SkyReels-I2V, we follow the image-to-video protocol of DanceGRPO and train on $27{,}000$ prompts constructed from ConsisID~\citep{yuan2025identity}; since the model conditions on a reference first frame, we pre-generate one reference image per prompt with FLUX~\citep{flux2024} at the reduced $640\times400$ (width $\times$ height) geometry used in this protocol and cache the resulting prompt--image pairs before training. For evaluation, Wan2.1-T2V-1.3B generates one video per prompt from the extended rewrites of the 946-prompt VBench suite. We evaluate these videos using the official VBench protocol (Huang et al., 2024). For SkyReels-I2V, we use the official VBench-I2V suite (Huang et al., 2026b), comprising 1,118 prompt–reference-image pairs, with reference images resized and center-cropped to the evaluation geometry.

\paragraph{\revred{Evaluation metrics}}
\revred{The cited VideoAlign scores for DanceGRPO and TaRoS in \cref{tab:vbench_t2v} are evaluated on $1.4$k iStock prompts~\citep{li2026rethinkingrewardsignalsvideo}, whereas ours use the $946$-prompt VBench suite. We retain these published values as references, not as a comparison on the same prompt set.}

\paragraph{\revred{Anchor evaluation}}\label{app:anchor_protocol}
\revred{For \cref{fig:anchor_ablation}, we select results from the same training window for all three anchor settings. The base reference without CFG uses $\beta=10^{-3}$, which performed well in this setting; the reference with CFG at scale $5$ uses $\beta=10^{-2}$. Training and inference use the learned velocity without CFG.}

\subsection{Text-to-Image (OCR) Task}
\label{app:t2i_impl}
For the OCR text-rendering task of \cref{tab:t2i}, we fine-tune Stable Diffusion 3.5-Medium~\citep{esser2024scaling} with LoRA at resolution $512\times512$, following the protocol of \citet{choi2026rethinkingdesignspacereinforcement}.

\paragraph{Hyperparameter settings}
Training is CFG-free throughout. Samples are drawn with a deterministic ODE sampler \citep{zhang2023fast, zhang2023gddim}. Among various ODE sampler, we use $10$-step DPM-Solver-2 flow ODE at guidance scale $1.0$, and the same guidance scale is used for the training and reference velocities. Each epoch draws $48$ prompts from the OCR prompt set with a group of $K{=}24$ images per prompt, followed by a single gradient update over the accumulated batch. We optimize with AdamW at learning rate $5\times10^{-5}$ with gradient clipping at norm $0.02$, use the adaptive ELBO weighting $w(t)$ (\cref{subapp:pg_loss}), and train for $500$ epochs on $6\times$H100 GPUs (${\sim}41$ hours). The per-sample reward weight is the mean-subtracted reward within each prompt group, and the anchor term uses the frozen reference velocity with strength $\beta=10^{-4}$. The full settings are collected in \cref{tab:hparams}.

\paragraph{Reward}
The training reward is an equal-weight sum of four components: an OCR accuracy reward computed by PaddleOCR on the rendered text, PickScore~\citep{kirstain2023pick} (rescaled by $1/26$ inside the scorer), CLIPScore~\citep{hessel2021clipscore}, and HPSv2.1~\citep{wu2023human}.

\paragraph{Evaluation}
Evaluation follows \citet{choi2026rethinkingdesignspacereinforcement}: images are sampled with $40$ steps, and we report OCR accuracy together with the \revred{preference and image-quality metrics} PickScore, ClipScore, HPSv2.1, Aesthetic score, and ImageReward in \cref{tab:t2i}.

\subsection{Video Reward}
\label{app:videoreward}
\textbf{VideoAlign \& HPSv3 Rewards}
Each reward component captures a distinct aspect of video quality. VideoAlign provides text alignment (TA), motion quality (MQ), and visual quality (VQ), which measure prompt faithfulness, temporal motion quality, and frame-level appearance, respectively. We replace VQ with HPSv3 to use a human-preference-based appearance score.

\textbf{Dynamic-tracking (DT) Reward}
Let $\{\bm{f}_1,\dots,\bm{f}_T\}$ be the generated frames, with $T=53$. We evaluate $n=8$ adjacent-frame pairs with starting indices $i_k=1+\operatorname{round}((k-1)(T-2)/(n-1))$, $k=1,\dots,n$. This uses $16$ distinct frames with a within-pair stride of one frame ($1/15$ second). RAFT operates at each model's generation resolution without spatial resizing. For each pair, we estimate $\bm{u}_k=\mathrm{RAFT}(\bm{f}_{i_k},\bm{f}_{i_k+1})\in\R^{2\times H\times W}$ and summarize it by the mean magnitude of its fastest pixels,
\begin{equation}\label{eq:dt-pair}
m_k=\frac{1}{|S_k|}\sum_{p\in S_k}\|\bm{u}_k(p)\|_2,
\end{equation}
where $p$ indexes pixels and $S_k$ collects the $\lfloor 0.05\,HW\rfloor$ pixels of largest flow magnitude, so that $m_k$ measures how fast the moving part of the scene travels between the two frames rather than how much of the frame moves. Each $m_k$ is then measured against the resolution-aware motion threshold $\tau=6\min(H,W)/256$, and averaged over the pairs:
\begin{equation}\label{eq:dt}
R_{\mathrm{dt}}(\bm{x}_0)=\frac{1}{n}\sum_{k=1}^{n}\min\!\left(\frac{m_k}{\tau},\,1\right)\in[0,1].
\tag{DT}
\end{equation}

\noindent\textbf{Second Version of the DT Reward for SkyReels-I2V}
For SkyReels-I2V training, we use a second version of the dynamic-tracking reward. We use a bounded window so that increasing flow magnitude does not continually increase the reward. Using the same $n=8$ frame pairs and RAFT inputs as DT, we compute the median-subtracted flow statistic $m_{\mathrm{fg}}=\frac{1}{n}\sum_{k=1}^{n}(m_k-\tilde m_k)$, where $\tilde m_k$ is the median flow magnitude of pair $k$, and apply
\begin{equation}\label{eq:dt2}
R_{\mathrm{dt2}}(\bm{x}_0)=\operatorname{clip}\!\left(\frac{m_{\mathrm{fg}}-\tau_{\mathrm{lo}}}{\tau_{\mathrm{mid}}-\tau_{\mathrm{lo}}},0,1\right)\operatorname{clip}\!\left(\frac{\tau_{\mathrm{hi}}-m_{\mathrm{fg}}}{\tau_{\mathrm{hi}}-\tau_{\mathrm{mid}}},0,1\right)\in[0,1],
\tag{DT2}
\end{equation}
The reward is bounded by $1$, decreases above $\tau_{\mathrm{mid}}$, and is zero outside $[\tau_{\mathrm{lo}},\tau_{\mathrm{hi}}]$. We use thresholds $(\tau_{\mathrm{lo}}, \tau_{\mathrm{mid}}, \tau_{\mathrm{hi}})=(3,5,8)$, in the same units as $m_k$. For SkyReels-I2V, $R_{\mathrm{dt2}}$ replaces $R_{\mathrm{dt}}$ in the reward mixture below at the same weight.

\revred{Unless otherwise stated, video training uses the third reward mixture. \Cref{subsec:ablation} compares these three mixtures, with reference-anchor settings specified in \cref{app:anchor_protocol}.}

\paragraph{Reward mixture}
\revred{The default video runs} use the same public reward mixture, a weighted sum of VideoAlign~\citep{liu2025videoalign} text alignment (TA, weight $1.5$) and motion quality (MQ, $1.0$), HPSv3~\citep{ma2025hpsv3} (general $0.1$ and percentile $0.1$), and our RAFT-based dynamic-tracking reward (DT, $0.7$); VideoAlign visual quality (VQ) is not used. \revred{Rewards are centered within each prompt group and normalized by the standard deviation across the full batch, as specified in \cref{sec:loss}.}
\revred{The further video experiments in \cref{app:further_video} use the model-specific reward settings stated there.}

\begin{table*}[t]
\centering
\caption{\textbf{Default hyperparameters} for the RVM (reward-weighted velocity matching) runs on each base model. All runs share the optimizer, DPM-2 ODE sampler, and CFG-free setting described above; the two video runs share the reward mixture of \cref{app:videoreward}, while the OCR run uses the reward of \cref{app:t2i_impl}.}
\label{tab:hparams}
\setlength{\tabcolsep}{8pt}
\renewcommand{\arraystretch}{1.1}
\begin{tabular}{l ccc}
\toprule
Hyperparameter & Wan2.1-T2V-1.3B & SkyReels-I2V & SD3.5-M (OCR) \\
\midrule
Task & text-to-video & image-to-video & text-to-image \\
Resolution ($H\times W$) & $480\times832$ & $400\times640$ & $512\times512$ \\
Frames / fps & $53$ / $15$ & $53$ / $15$ & --- \\
Sampler & DPM-2 ODE & DPM-2 ODE & DPM-2 ODE \\
Sampling steps & $16$ & $16$ & $10$ \\
Flow-matching shift & $8.0$ & --- & --- \\
LoRA rank $r$ & $128$ & $32$ & $32$ \\
LoRA $\alpha$ & $64$ & $64$ & $64$ \\
Learning rate & $5\times10^{-5}$ & $5\times10^{-5}$ & $5\times10^{-5}$ \\
Group size $K$ & $8$ & $8$ & $24$ \\
Prompts / iteration & $32$ & $16$ & $48$ \\
Optimizer updates / epoch & $2$ & $2$ & $1$ \\
Samples / update & $128$ & $64$ & $1152$ \\
Training epochs & $90$ & $40$ & $500$ \\
Anchor strength $\beta$ (default) & $0$ & $0$ & $10^{-4}$ \\
Gradient clip (norm) & $1.0$ & $1.0$ & $0.02$ \\
Guidance scale (CFG) & $1.0$ & $1.0$ & $1.0$ \\
GPUs & $8$ & $8$ & $6$ \\
\bottomrule
\end{tabular}
\end{table*}

\FloatBarrier
\section{Retaining Performance with Few-Step Inference}\label{app:few_step}

\paragraph{Interpreting few-step robustness}
RVM's robustness to fewer inference steps may be related to its endpoint-based regression structure. For $t>0$, the target satisfies $\bm{v}_{\mathrm{tgt}}^i=(\bm{x}_t^i-\bm{x}_0^i)/t$, so velocity matching equivalently regresses $\bm{x}_t^i-t\bm{v}_\theta(\bm{x}_t^i,t)$ toward $\bm{x}_0^i$. Ordinary squared-error regression predicts the conditional barycenter of compatible endpoints, while RVM reshapes this regression through reward weighting. We conjecture that these endpoint-directed updates may simplify denoising trajectories, helping retain performance under coarser discretization.

\paragraph{Evaluation setup}
We evaluate whether RVM retains generation performance when inference uses substantially fewer sampling steps than training. Training samples are generated with the standard $16$-step sampler; the $5$- and $4$-step schedules are used only at inference, without further training. We compare RVM without CFG against the CFG-guided base model under identical inference schedules.

\paragraph{Discretization}
The endpoint interpretation also motivates preserving a lower-noise evaluation before the terminal update. To see this, a final Euler step from $t$ to $0$ gives $\bm{x}_t-t\bm{v}_\theta(\bm{x}_t,t)$, corresponding to a conditional barycenter under ordinary squared-error matching. At high noise, averaging over distinct compatible endpoints can obscure sample-specific details; evaluating at lower noise can reduce this ambiguity. Guided by this intuition, we sub-sample the $16$-step noise grid while retaining its final lower-noise transition:
\begin{equation}
\begin{aligned}
4\text{ steps}:&\quad
(0.999,\,0.889,\,0.727,\,0.348,\,0),\\
5\text{ steps}:&\quad
(0.999,\,0.960,\,0.889,\,0.727,\,0.348,\,0).
\end{aligned}
\end{equation}
Both schedules preserve the final transition from $0.348$ to $0$. This interpretation motivates our schedule choice; we do not separately test the proposed mechanism or the optimality of these timesteps.

% Keep the wrapped plot and its caption clear of the page footer.
\noindent\begin{minipage}{\textwidth}
\begin{wrapfigure}{r}{0.50\textwidth}
    \vspace{-\baselineskip}
    \centering
    \includegraphics[width=\linewidth]{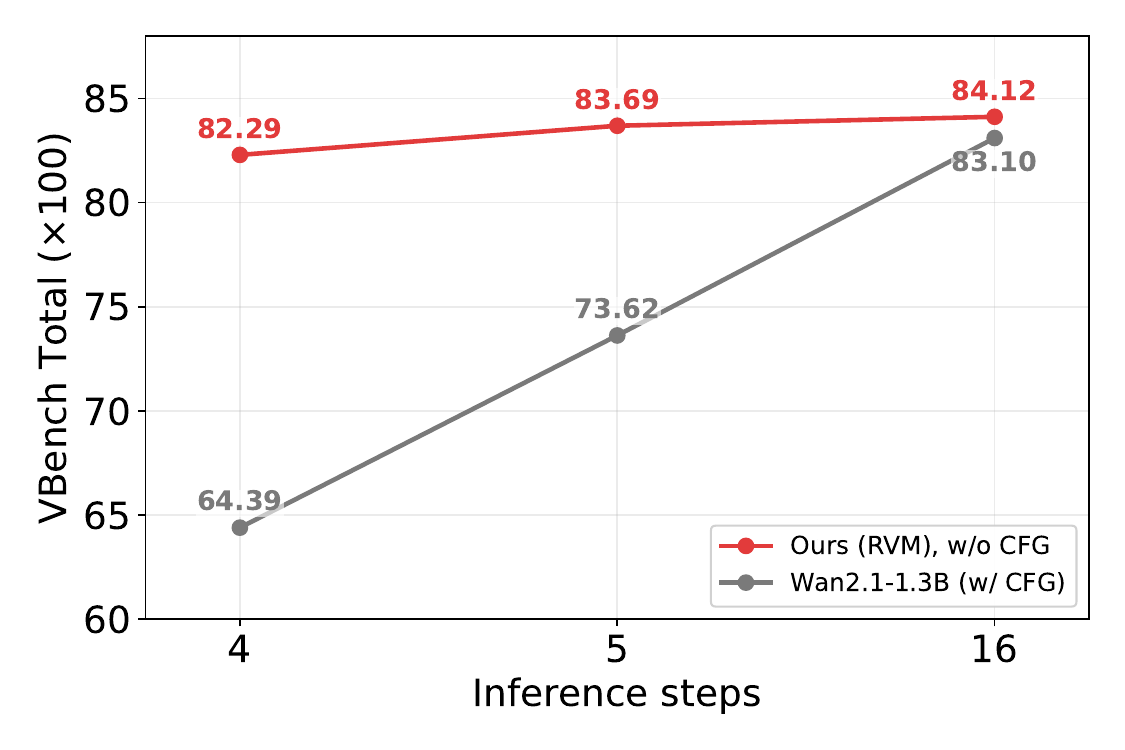}
    \vspace{-20pt}
    \caption{{\textbf{Few-step inference on Wan2.1-T2V-1.3B.} VBench Total against the number of inference steps, under the same timestep schedules.}}
    \label{fig:step_ablation}
\end{wrapfigure}
\paragraph{Results}
\Cref{fig:step_ablation} shows that RVM retains much of its generation quality\revred{, as measured by VBench Total,} as the inference budget decreases. Its VBench Total decreases from $84.12$ at $16$ steps to $83.69$ at $5$ steps and $82.29$ at $4$ steps, corresponding to drops of $0.43$ and $1.83$ points. In comparison, the CFG-guided base model decreases from $83.10$ to $73.62$ and $64.39$, with drops of $9.48$ and $18.71$ points. These results demonstrate substantially greater robustness of RVM \revred{in VBench Total} to reduced inference steps under the evaluated schedules. They are consistent with the trajectory-simplification hypothesis, although they do not directly measure trajectory geometry or isolate the effect of timestep placement.
\par
\ifnum\value{WF@wrappedlines}>1
    \vspace{\dimexpr\value{WF@wrappedlines}\baselineskip-\baselineskip\relax}
\fi
\WFclear
\end{minipage}

\begingroup
\color{black}
\hypersetup{linkcolor=black,citecolor=black,urlcolor=black}
\section{Further Video Fine-tuning Results}\label{app:further_video}

We extend the video experiments to SkyReels-V1-I2V at its native spatial high resolution.

\subsection{Experimental Setup}

\paragraph{SkyReels at native spatial resolution}\label{app:skyreels_highres}
The main-text experiments follow DanceGRPO~\citep{xue2025dancegrpo} at $400\times640$, where the guided base scores lower than at SkyReels-V1's native $544\times960$ resolution~{\color{black}\hypersetup{citecolor=black}\citep{SkyReelsV1}}. We return to the native spatial resolution to evaluate RVM in a larger video space. Relative to \cref{app:videoreward}, we set the MQ weight to $0.5$ and whiten reward components before mixing.

\subsection{Results}

\paragraph{RVM improves SkyReels at native spatial resolution}
\Cref{tab:skyreels_highres} shows that RVM improves Overall from the guided base's $84.12$ to $87.58$ and Dynamic Degree from $38.62$ to $49.19$, while also improving Quality and I2V scores. RVM generates without CFG at the model's native spatial resolution.

\begin{table}[t]
\centering
\color{black}\arrayrulecolor{black}
\captionsetup{labelfont={color=black},textfont={color=black}}
\caption{\textbf{Evaluation results on SkyReels-I2V at native spatial resolution} We report VBench-I2V scores on the VBench-I2V dataset. \texttt{Dyn.\ Degree} is the VBench dynamic-degree dimension, one of the seven quality dimensions already averaged into \texttt{Quality}.}
\label{tab:skyreels_highres}
\normalsize
\setlength{\tabcolsep}{5pt}
\renewcommand{\arraystretch}{1}
\begin{tabular*}{\textwidth}{@{\extracolsep{\fill}}l cccc}
\toprule
& \multicolumn{4}{c}{VBench-I2V} \\
\cmidrule(lr){2-5}
Method & Dyn.\ Degree & Quality & I2V & Overall \\
\midrule
SkyReels-I2V-V1 (w/o CFG) & \cellbest$54.10$ & $75.85$ & $89.74$ & $82.80$ \\
SkyReels-I2V-V1 (w/ CFG) & $38.62$ & \cellsecond$76.92$ & \cellsecond$91.32$ & \cellsecond$84.12$ \\
\midrule
\revred{RVM} & \cellsecond$49.19$ & \cellbest$80.50$ & \cellbest$94.67$ & \cellbest$87.58$ \\
\bottomrule
\end{tabular*}
\end{table}

\FloatBarrier
\arrayrulecolor{black}
\endgroup

\section{Additional Text-to-Image Experimental Results}\label{app:t2i}

{In addition to the OCR task reported in \cref{tab:t2i}, we evaluate the PickScore task of \citet{choi2026rethinkingdesignspacereinforcement} on DrawBench, \revred{fine-tuning SD3.5-M on the mixture of PickScore, CLIPScore, and HPSv2.1}. As shown in \cref{tab:t2i_drawbench}, Ours (RVM) remains competitive.}

\begin{table}[h]
\centering
\caption{{\textbf{Evaluation Results on SD3.5-M (DrawBench)} We report the text-to-image PickScore fine-tuning task, evaluated on DrawBench with the protocol of \citet{choi2026rethinkingdesignspacereinforcement}. Training uses PickScore, CLIPScore, and HPSv2.1. Blue shading marks the \ranklegend highest scores per column, including ties. $^\dagger$cited numbers from \citet{choi2026rethinkingdesignspacereinforcement}.}}
\label{tab:t2i_drawbench}
\normalsize\color{black}\arrayrulecolor{black}
\setlength{\tabcolsep}{5pt}
\renewcommand{\arraystretch}{1}
\begin{tabular*}{\textwidth}{@{\extracolsep{\fill}}l ccccc}
\toprule
Method & PickScore & ClipScore & HPSv2.1 & Aesthetic & ImgRwd \\
\midrule
{FlowGRPO$^\dagger$} & $23.50$ & $0.280$ & $0.316$ & $5.90$ & $1.29$ \\
{DiffusionNFT$^\dagger$} & \cellsecond$23.61$ & $0.288$ & \cellbest$0.344$ & \cellsecond$6.04$ & \cellbest$1.46$ \\
{PEPG$^\dagger$} & \cellbest$23.68$ & \cellbest$0.296$ & $0.325$ & \cellbest$6.06$ & \cellsecond$1.45$ \\
{\revred{RVM}} & $23.30$ & \cellsecond$0.289$ & \cellsecond$0.333$ & $5.82$ & $1.38$ \\
\bottomrule
\end{tabular*}
\arrayrulecolor{black}
\end{table}

\section{Additional Experimental Results}\label{app:results}

\subsection{Per-reward training curves}\label{app:curves}
{\cref{fig:curves_all} shows training curves for VideoAlign TA and MQ, HPSv3, and VBench Dynamic Degree. The velocity-matching objectives improve every component throughout training, while FlowGRPO stagnates or degrades, consistent with \cref{tab:vbench_t2v}.}

\begin{figure*}[h]
    \centering
    \begin{subfigure}[t]{0.49\textwidth}
      \centering
      \includegraphics[width=\linewidth]{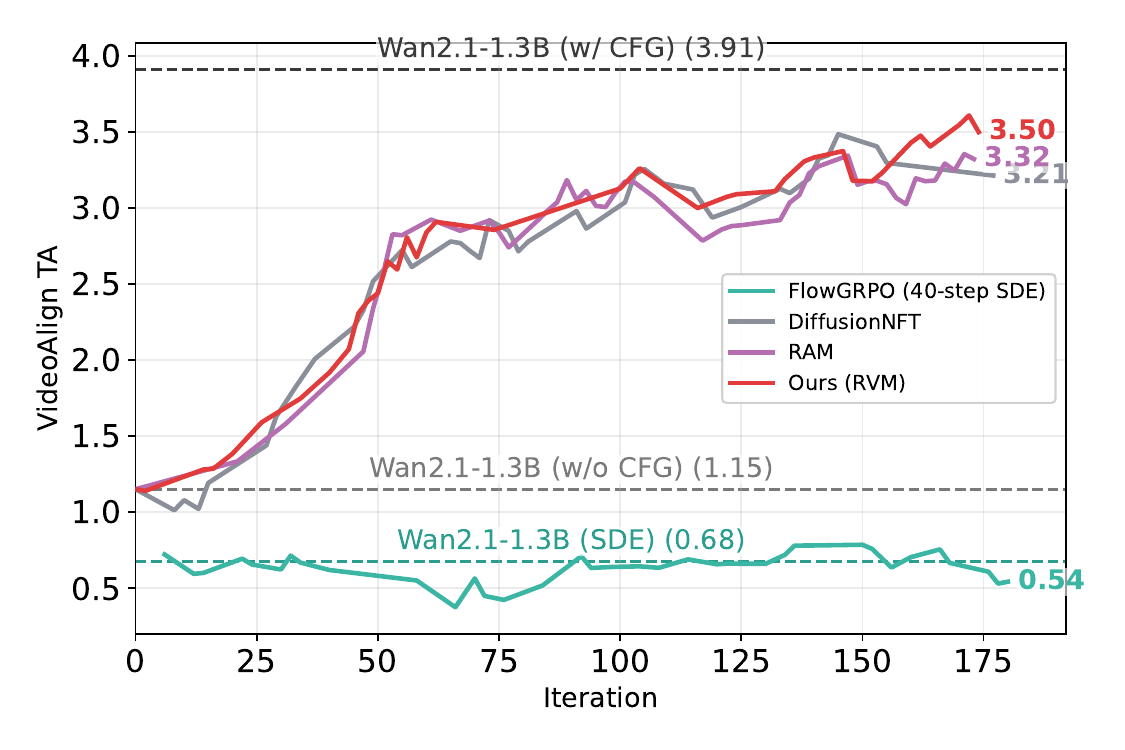}
      \caption{VideoAlign text alignment (TA)}
    \end{subfigure}
    \hfill
    \begin{subfigure}[t]{0.49\textwidth}
      \centering
      \includegraphics[width=\linewidth]{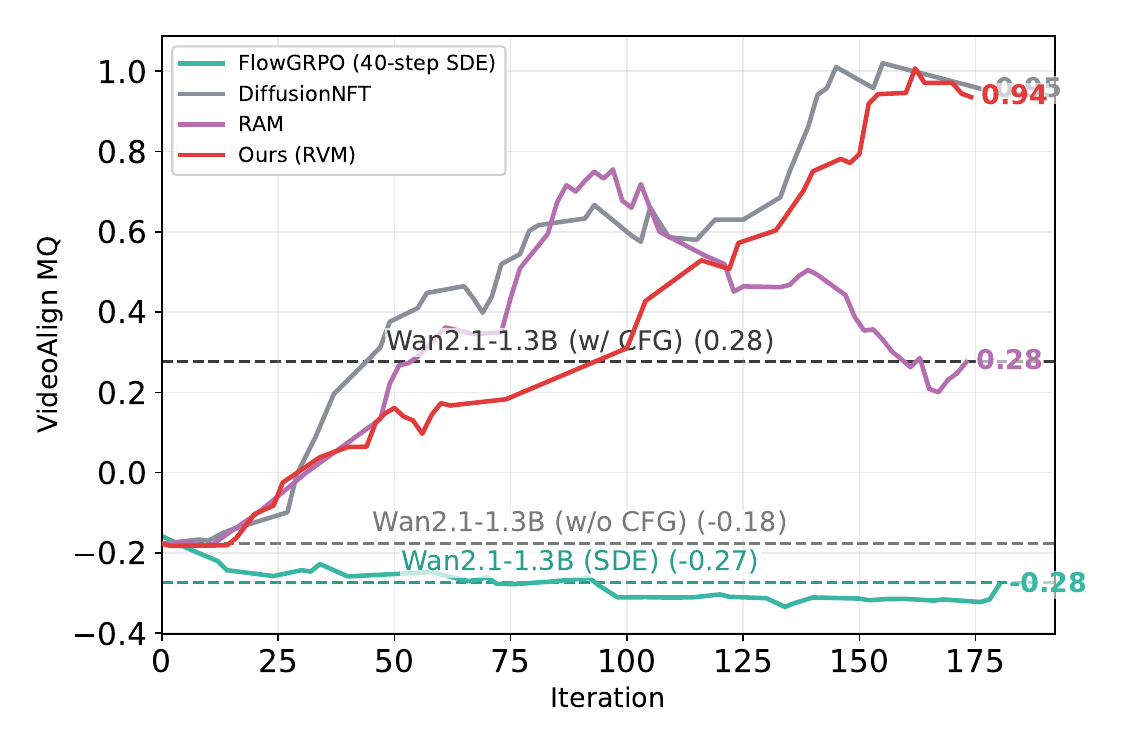}
      \caption{VideoAlign motion quality (MQ)}
    \end{subfigure}

    \vspace{6pt}
    \begin{subfigure}[t]{0.49\textwidth}
      \centering
      \includegraphics[width=\linewidth]{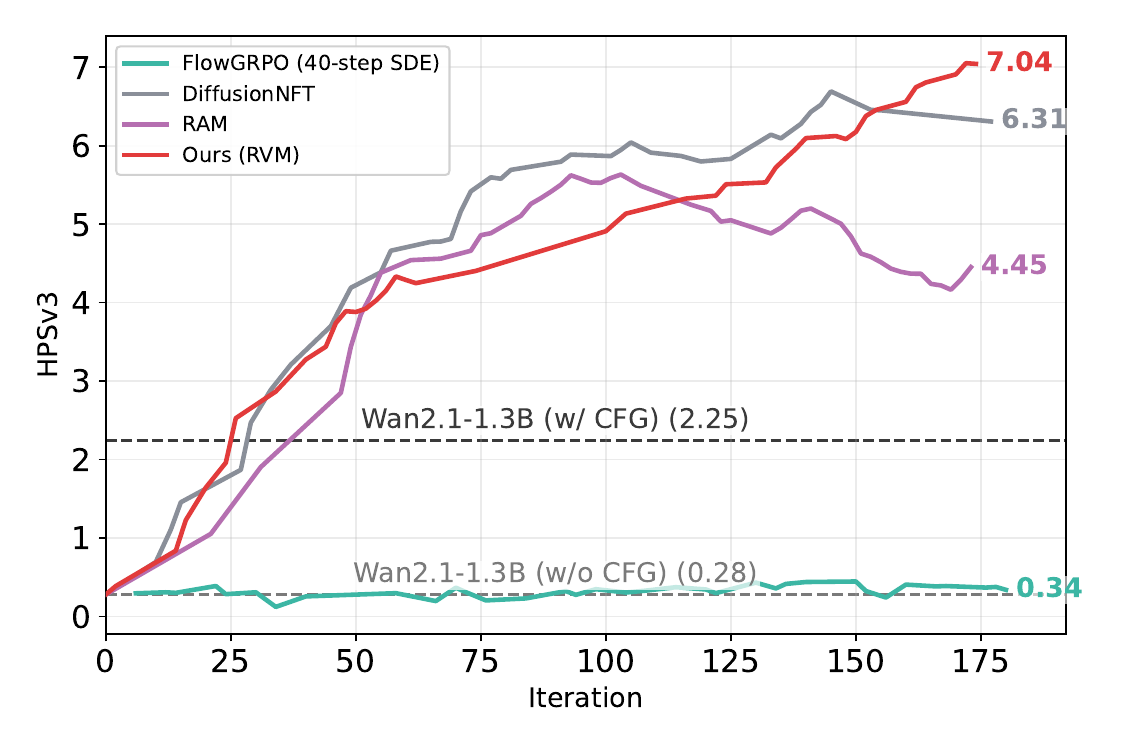}
      \caption{HPSv3}
    \end{subfigure}
    \hfill
    \begin{subfigure}[t]{0.49\textwidth}
      \centering
      \includegraphics[width=\linewidth]{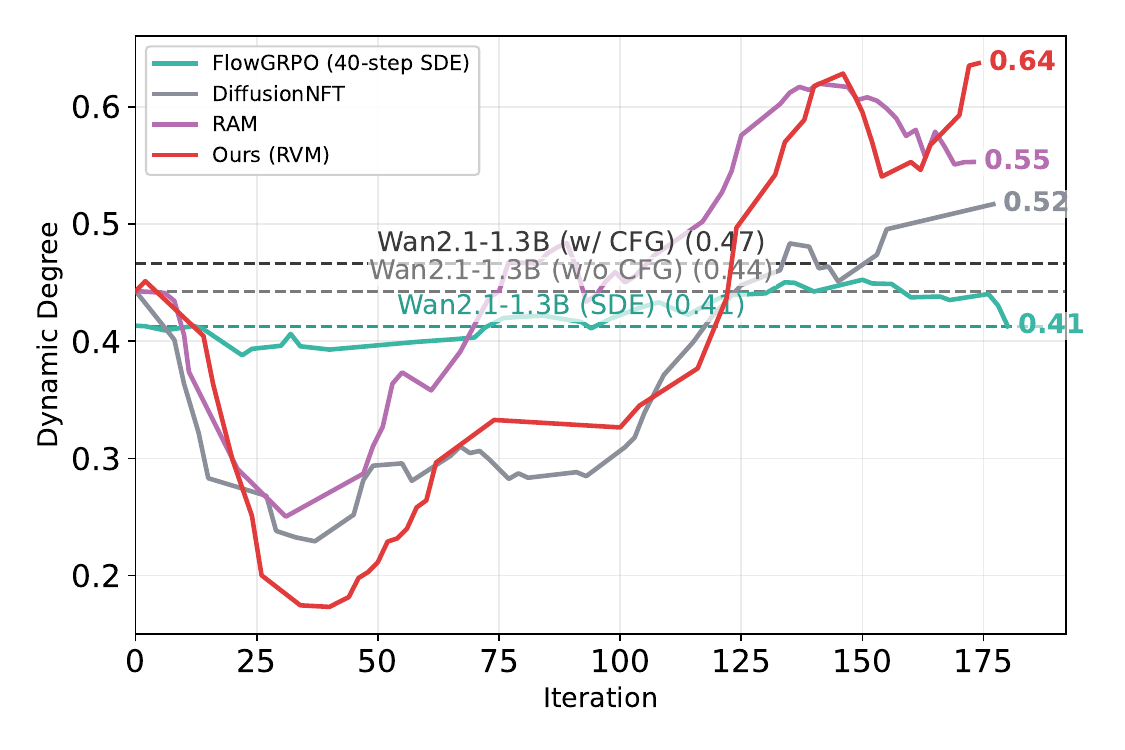}
      \caption{VBench dynamic degree}
    \end{subfigure}
    \caption{{\textbf{Per-reward training curves on Wan2.1-T2V-1.3B} Each panel shows one reward component across training iterations; dashed lines mark the untrained base.}}
    \label{fig:curves_all}
\end{figure*}

\FloatBarrier
\subsection{Full per-dimension VBench results}\label{app:vbench_full}
{\cref{tab:vbench_t2v_full,tab:vbench_i2v_full} expand the aggregate scores of \cref{tab:vbench_t2v,tab:vbench_i2v} into all 16 VBench and 10 VBench-I2V dimensions under the same protocol. \revred{\Cref{tab:vbench_t2v_full} also includes the CFG-reference configuration from \cref{app:anchor_protocol}. On Wan2.1-T2V-1.3B, the default RVM run retains the highest Quality score, while the CFG-reference run obtains the highest Overall score in this table, $84.17$, with Semantic $78.88$. The CFG-guided base retains the highest Semantic score. The two RVM columns use different training data and reward scaling, as detailed in \cref{app:anchor_protocol}.} On SkyReels-I2V, the three velocity-matching methods reach similar Overall scores, and Ours (RVM) is the only fine-tuned method that raises dynamic degree above the base.}

\begin{table*}[h]
\centering
\caption{{\textbf{Full per-dimension VBench Results on Wan2.1-T2V-1.3B} We report the per-dimension expansion of \cref{tab:vbench_t2v}; all scores $\times100$. $\dagger$ columns are cited numbers.}}
\label{tab:vbench_t2v_full}
\setlength{\tabcolsep}{4pt}
\renewcommand{\arraystretch}{1}
\begin{adjustbox}{max width=\textwidth}
% GRPO column commented out for now (confusing). GRPO per-dim (Quality): 97.68 / 96.62 / 98.36 / 98.33 / 59.72 / 66.27 / 69.82; (Semantic): 91.00 / 19.55 / 23.11 / 25.44 / 95.25 / 87.20 / 77.37 / 77.00 / 54.22; aggregates Q7 84.52 / S9 77.28 / O16 83.07.
\begin{tabular}{l l c c c c c c c c c c}
\toprule
& Dimension & \revred{\shortstack{Wan2.1-1.3B\\(w/o CFG)}} & \revred{\shortstack{Wan2.1-1.3B\\(w/ CFG)}} & FlowGRPO & DanceGRPO$^\dagger$ & TaRoS$^\dagger$ & TaRoS-72B$^\dagger$ & {DiffusionNFT} & RAM & {\revred{RVM}} & \revred{\shortstack{RVM\\(CFG ref.)}} \\
\midrule
\multirow{7}{*}{\rotatebox{90}{Quality}}
 & subject consistency & $91.78$ & $95.02$ & $91.29$ & $95.55$ & $95.22$ & $95.19$ & $97.52$ & $96.42$ & $98.31$ & \revred{$97.85$} \\
 & background consistency & $95.14$ & $97.72$ & $94.77$ & $96.83$ & $96.77$ & $96.85$ & $96.06$ & $95.25$ & $96.22$ & \revred{$96.79$} \\
 & temporal flickering & $98.95$ & $99.54$ & $98.95$ & $99.42$ & $99.41$ & $99.35$ & $97.93$ & $98.62$ & $97.25$ & \revred{$98.26$} \\
 & motion smoothness & $98.48$ & $98.27$ & $98.19$ & $98.15$ & $98.28$ & $98.23$ & $98.25$ & $98.43$ & $98.44$ & \revred{$98.22$} \\
 & dynamic degree & $54.17$ & $65.28$ & $55.56$ & $58.33$ & $57.67$ & $58.33$ & $63.89$ & $61.11$ & $75.00$ & \revred{$70.83$} \\
 & aesthetic quality & $52.43$ & $63.77$ & $52.74$ & $58.99$ & $60.88$ & $62.64$ & $65.47$ & $63.13$ & $69.09$ & \revred{$67.92$} \\
 & imaging quality & $54.75$ & $62.97$ & $56.29$ & $66.61$ & $67.92$ & $68.49$ & $69.92$ & $66.99$ & $72.02$ & \revred{$69.16$} \\
\midrule
\multirow{9}{*}{\rotatebox{90}{Semantic}}
 & human action & $79.00$ & $96.00$ & $78.00$ & $73.00$ & $74.00$ & $76.00$ & $95.00$ & $91.00$ & $95.00$ & \revred{$93.00$} \\
 & appearance style & $21.05$ & $21.40$ & $20.78$ & $20.70$ & $20.82$ & $21.26$ & $19.10$ & $19.74$ & $19.14$ & \revred{$19.94$} \\
 & temporal style & $22.47$ & $24.62$ & $22.35$ & $23.51$ & $23.29$ & $23.15$ & $23.28$ & $23.24$ & $22.97$ & \revred{$23.71$} \\
 & overall consistency & $23.61$ & $26.91$ & $23.10$ & $23.68$ & $23.67$ & $23.67$ & $25.44$ & $25.36$ & $25.46$ & \revred{$26.23$} \\
 & object class & $80.38$ & $93.91$ & $79.51$ & $77.68$ & $76.42$ & $74.21$ & $96.68$ & $96.68$ & $92.80$ & \revred{$96.28$} \\
 & multiple objects & $51.45$ & $82.55$ & $48.25$ & $57.54$ & $59.73$ & $63.49$ & $88.49$ & $81.55$ & $82.16$ & \revred{$86.51$} \\
 & color & $72.44$ & $91.40$ & $67.70$ & $90.05$ & $88.26$ & $89.15$ & $72.96$ & $78.16$ & $75.94$ & \revred{$83.11$} \\
 & spatial relationship & $55.19$ & $78.60$ & $58.74$ & $62.89$ & $67.16$ & $72.00$ & $75.18$ & $73.49$ & $73.77$ & \revred{$75.96$} \\
 & scene & $43.46$ & $54.94$ & $44.48$ & $25.79$ & $24.62$ & $25.94$ & $54.72$ & $52.25$ & $55.01$ & \revred{$56.03$} \\
\midrule
\multicolumn{2}{l}{Quality (7)} & $78.59$ & $83.72$ & $78.67$ & $82.81$ & $83.23$ & $83.66$ & $84.37$ & $83.35$ & $86.09$ & \revred{$85.49$} \\
\multicolumn{2}{l}{Semantic (9)} & $65.73$ & $80.64$ & $64.87$ & $66.06$ & $66.26$ & $68.15$ & $77.28$ & $76.33$ & $76.28$ & \revred{$78.88$} \\
\multicolumn{2}{l}{Overall (16)} & $76.02$ & $83.10$ & $75.91$ & $79.46$ & $79.84$ & $80.56$ & $82.95$ & $81.95$ & $84.13$ & \revred{$84.17$} \\
\bottomrule
\end{tabular}
\end{adjustbox}
\end{table*}

\begin{table*}[h]
\centering
\caption{{\textbf{Full per-dimension VBench-I2V Results on SkyReels-I2V} We report the per-dimension expansion of \cref{tab:vbench_i2v}; all scores $\times100$. SkyReels-I2V-V1 is the untrained base model; $^\S$the same base sampled with true CFG (scale $6$) and the official SkyReels-V1 negative prompt.}}
\label{tab:vbench_i2v_full}
\setlength{\tabcolsep}{6pt}
\renewcommand{\arraystretch}{1}
\begin{adjustbox}{max width=\textwidth}
\begin{tabular}{l l c c c c c c}
\toprule
& Dimension & {SkyReels-I2V-V1} & {V1 (CFG$=6$)$^\S$} & FlowGRPO & DiffusionNFT & RAM & {\revred{RVM}} \\
\midrule
\multirow{3}{*}{\rotatebox{90}{I2V}}
& i2v subject              & $91.50$ & $92.00$ & $88.53$ & $96.97$ & $97.03$ & $95.96$ \\
& i2v background           & $93.48$ & $94.56$ & $93.18$ & $97.15$ & $97.65$ & $96.59$ \\
& camera motion            & $31.32$ & $33.68$ & $32.63$ & $30.93$ & $29.88$ & $30.93$ \\
\midrule
\multirow{7}{*}{\rotatebox{90}{Quality}}
& subject consistency      & $86.17$ & $87.85$ & $79.63$ & $94.96$ & $95.29$ & $93.27$ \\
& background consistency   & $91.56$ & $93.83$ & $91.33$ & $96.27$ & $96.44$ & $95.32$ \\
& aesthetic quality        & $53.32$ & $56.55$ & $48.61$ & $58.42$ & $58.97$ & $57.77$ \\
& imaging quality          & $64.13$ & $63.24$ & $58.31$ & $68.26$ & $68.62$ & $68.00$ \\
& temporal flickering      & $96.22$ & $96.87$ & $96.58$ & $97.44$ & $97.28$ & $96.24$ \\
& motion smoothness        & $97.85$ & $97.80$ & $97.59$ & $98.53$ & $98.47$ & $97.94$ \\
& dynamic degree           & $64.63$ & $47.97$ & $41.46$ & $47.56$ & $51.22$ & $72.36$ \\
\midrule
\multicolumn{2}{l}{I2V}      & $87.79$ & $88.87$ & $86.00$ & $93.18$ & $93.49$ & $92.26$ \\
\multicolumn{2}{l}{Quality}  & $75.57$ & $75.36$ & $69.93$ & $79.15$ & $79.72$ & $80.28$ \\
\multicolumn{2}{l}{Overall}  & $81.68$ & $82.12$ & $77.97$ & $86.16$ & $86.61$ & $86.27$ \\
\bottomrule
\end{tabular}
\end{adjustbox}
\end{table*}

\FloatBarrier
\subsection{Additional qualitative comparisons}\label{app:qualitative}
\cref{fig:app_couple} extends \cref{fig:all_comp} to the full set of methods on Wan2.1-T2V-1.3B. \cref{fig:app_bird,fig:app_corgi} provide four additional Wan2.1 comparisons with the base model, and \cref{fig:app_guitar,fig:app_sky_more} give six further SkyReels-I2V examples. All videos use the protocol of \cref{tab:vbench_t2v,tab:vbench_i2v}, and the prompt is shown once below the bottom row.

\begin{figure*}[h]
    \centering
    \includegraphics[width=0.92\linewidth]{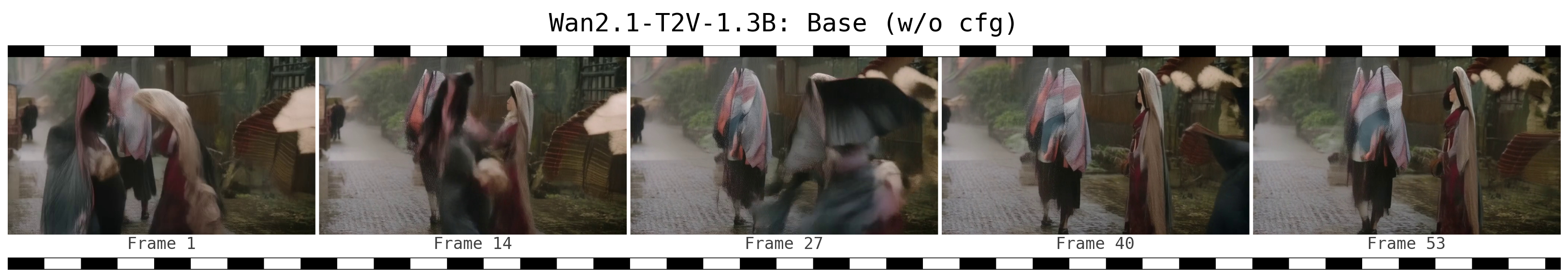}\\[1pt]
    \includegraphics[width=0.92\linewidth]{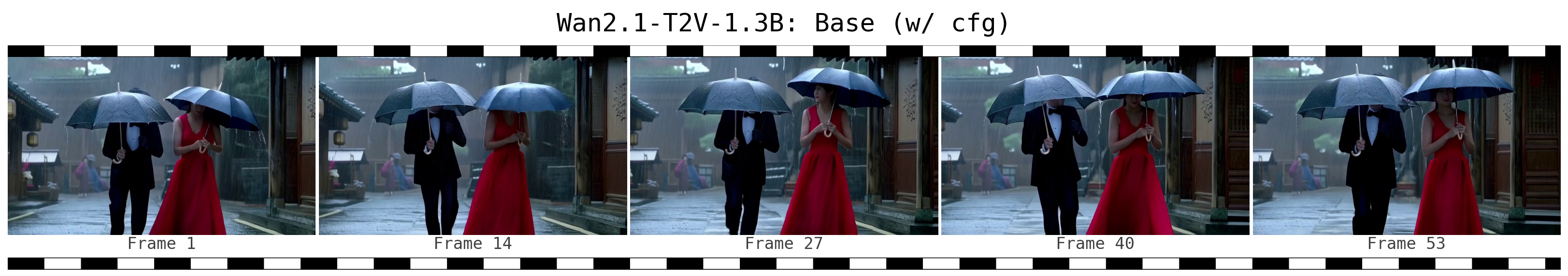}\\[1pt]
    \includegraphics[width=0.92\linewidth]{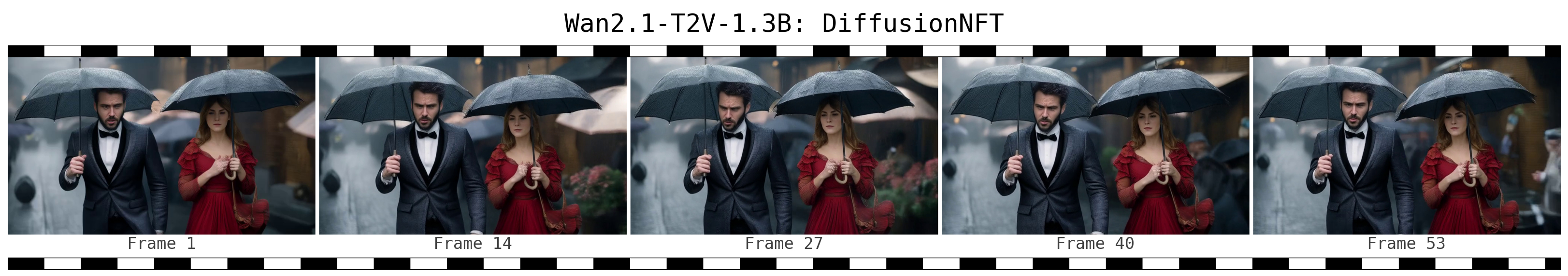}\\[1pt]
    \includegraphics[width=0.92\linewidth]{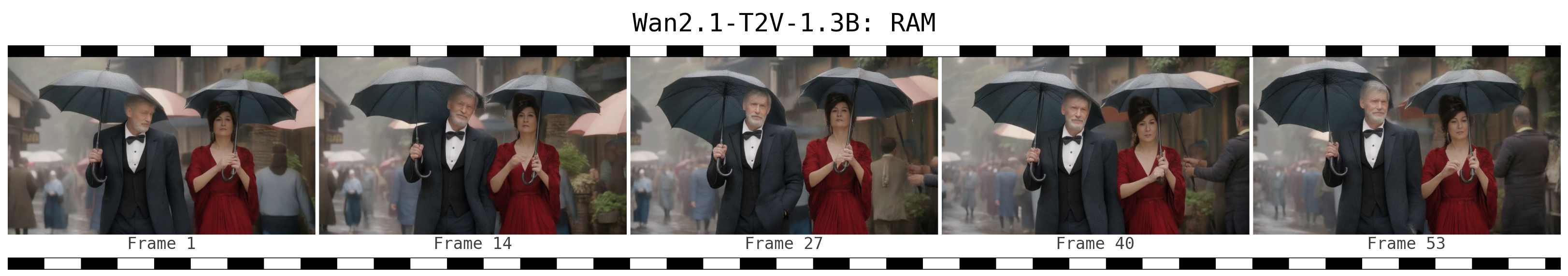}\\[1pt]
    \includegraphics[width=0.92\linewidth]{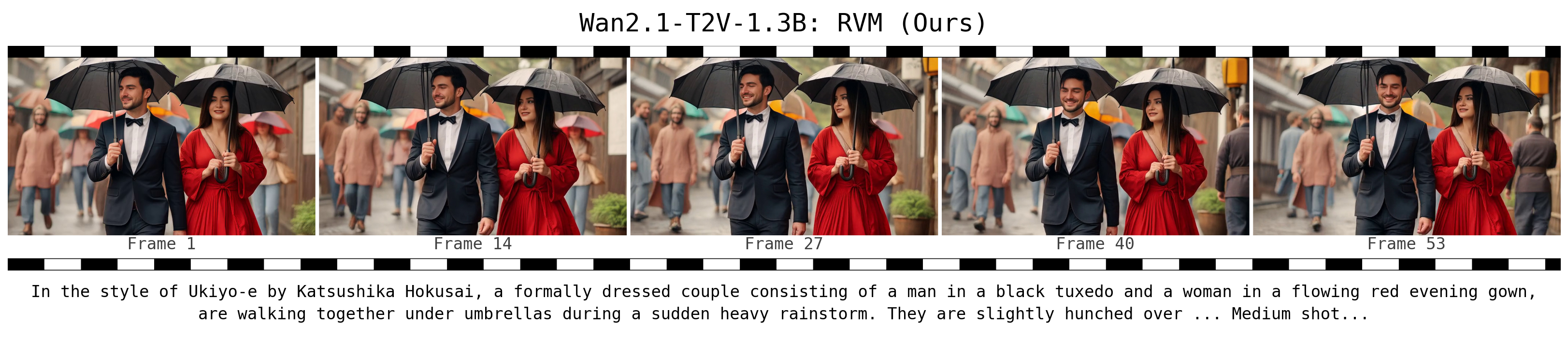}
    \caption{\textbf{Additional qualitative comparison on Wan2.1-T2V-1.3B} Frames $1, 14, 27, 40, 53$ of each method for the same prompt; top to bottom: Wan2.1-1.3B (w/o CFG), Wan2.1-1.3B (w/ CFG), DiffusionNFT, RAM, Ours (RVM).}
    \label{fig:app_couple}
\end{figure*}

\begin{figure*}[h]
    \centering
    \includegraphics[width=0.92\linewidth]{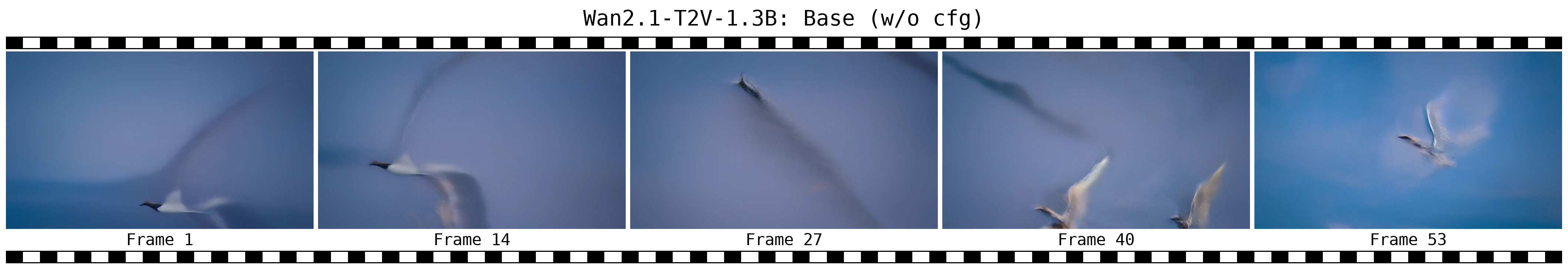}\\[1pt]
    \includegraphics[width=0.92\linewidth]{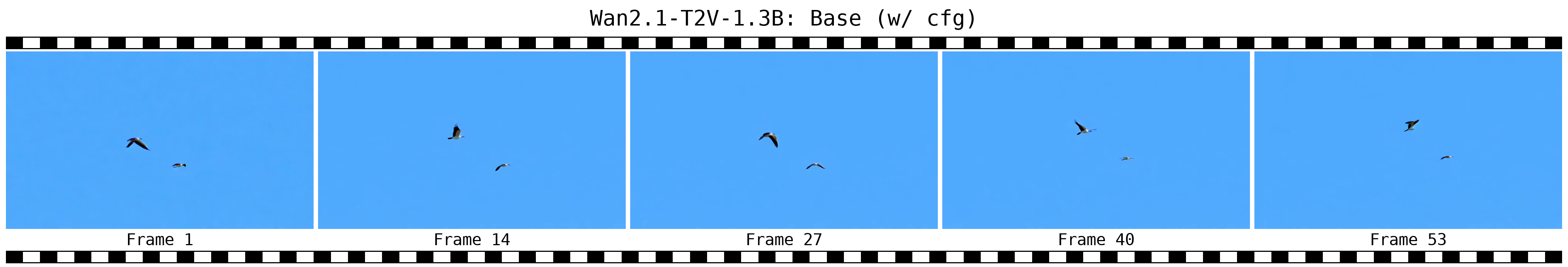}\\[1pt]
    \includegraphics[width=0.92\linewidth]{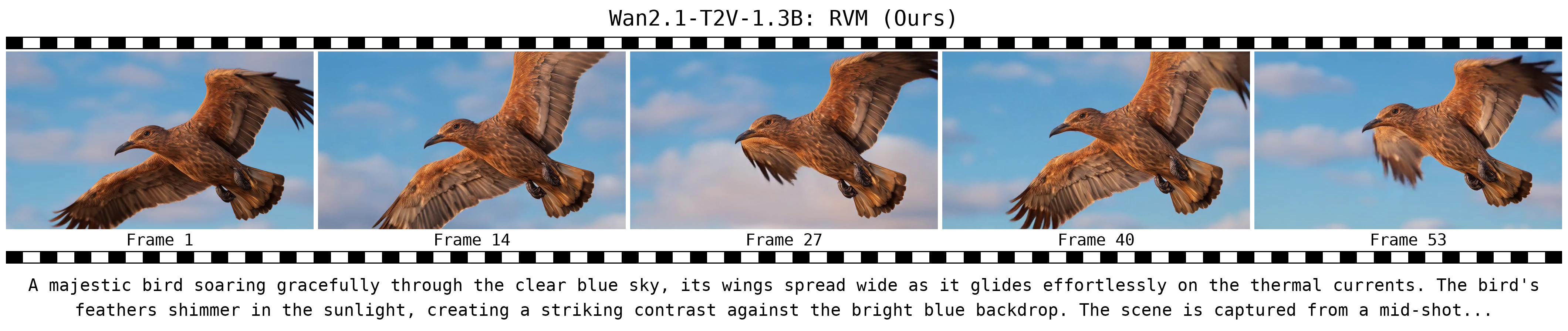}\\[6pt]
    \includegraphics[width=0.92\linewidth]{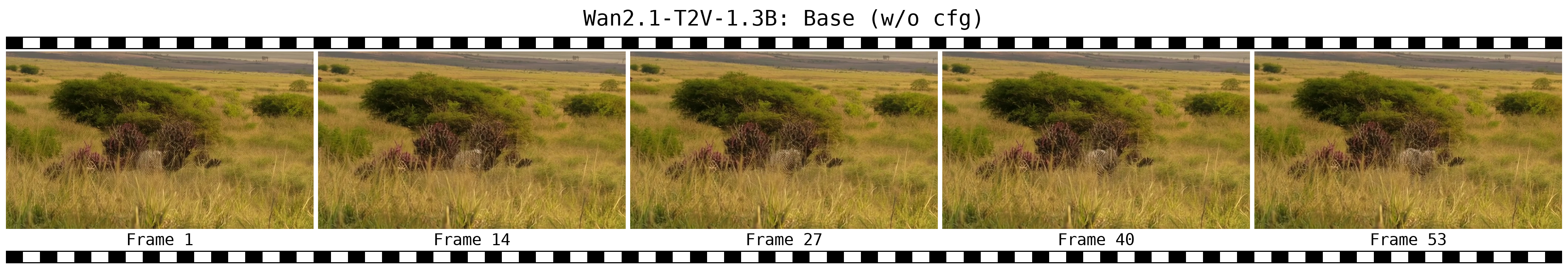}\\[1pt]
    \includegraphics[width=0.92\linewidth]{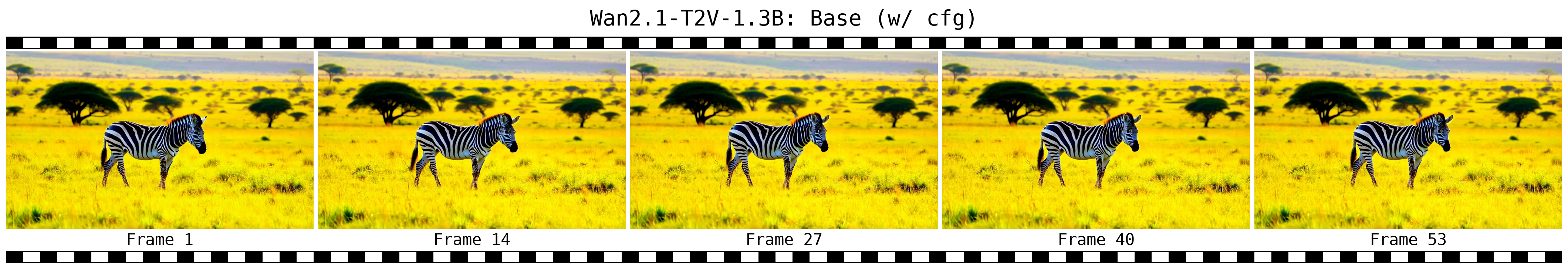}\\[1pt]
    \includegraphics[width=0.92\linewidth]{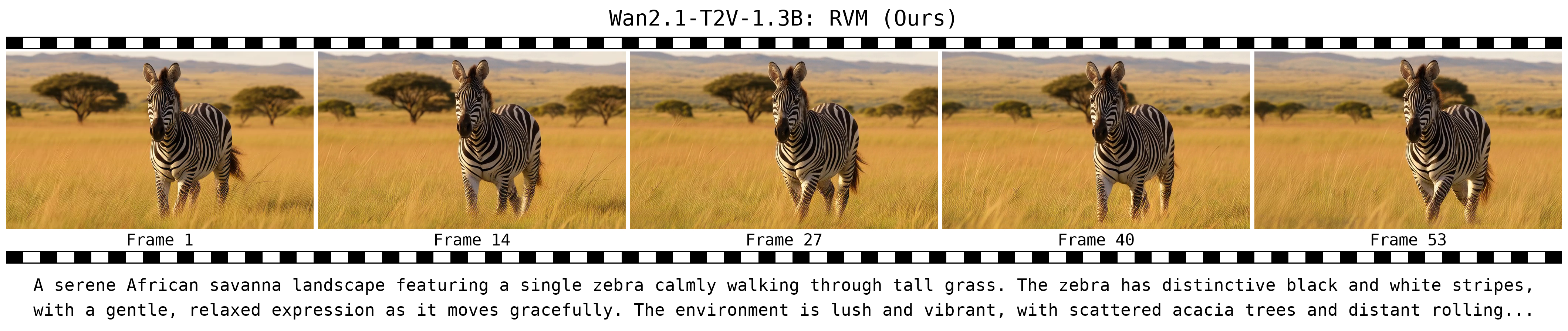}
    \caption{\textbf{Additional qualitative comparison on Wan2.1-T2V-1.3B}. Frames $1, 14, 27, 40, 53$ of each method for two prompts; within each group of three rows, top to bottom: Wan2.1-1.3B (w/o CFG), Wan2.1-1.3B (w/ CFG), Ours (RVM).}
    \label{fig:app_bird}
\end{figure*}

\begin{figure*}[h]
    \centering
    \includegraphics[width=0.92\linewidth]{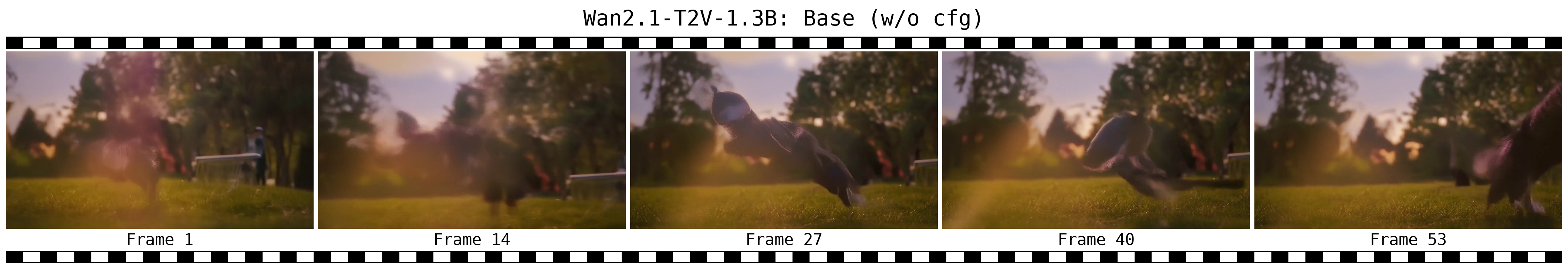}\\[1pt]
    \includegraphics[width=0.92\linewidth]{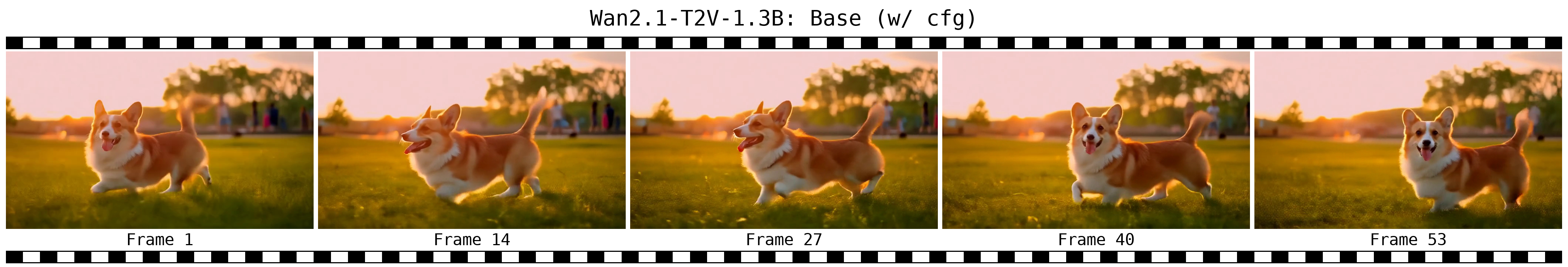}\\[1pt]
    \includegraphics[width=0.92\linewidth]{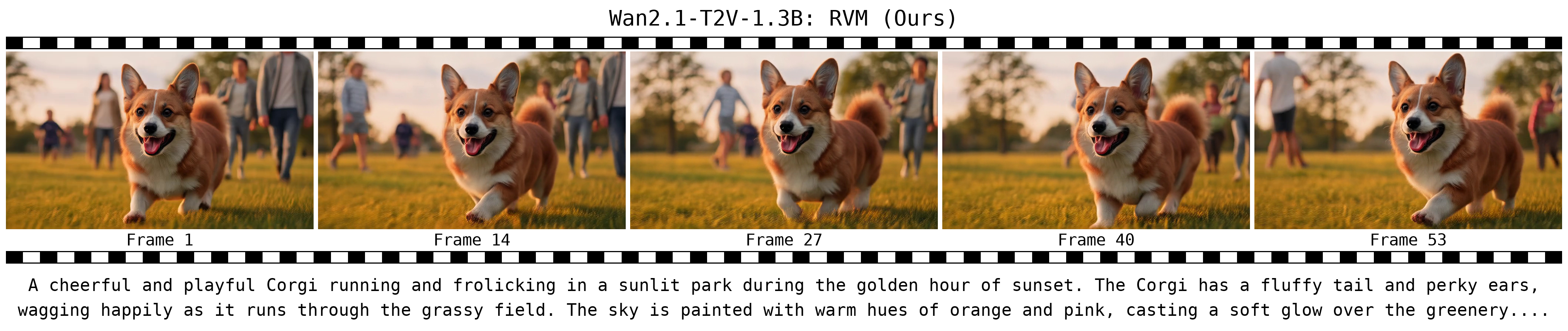}\\[6pt]
    \includegraphics[width=0.92\linewidth]{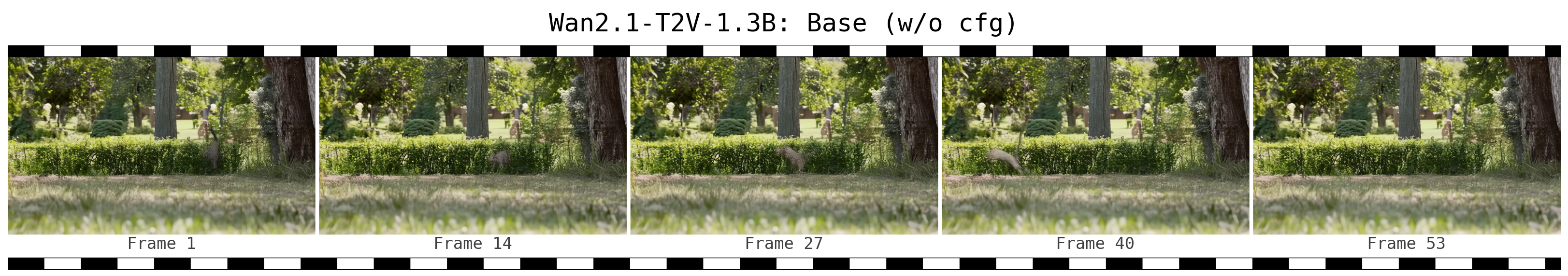}\\[1pt]
    \includegraphics[width=0.92\linewidth]{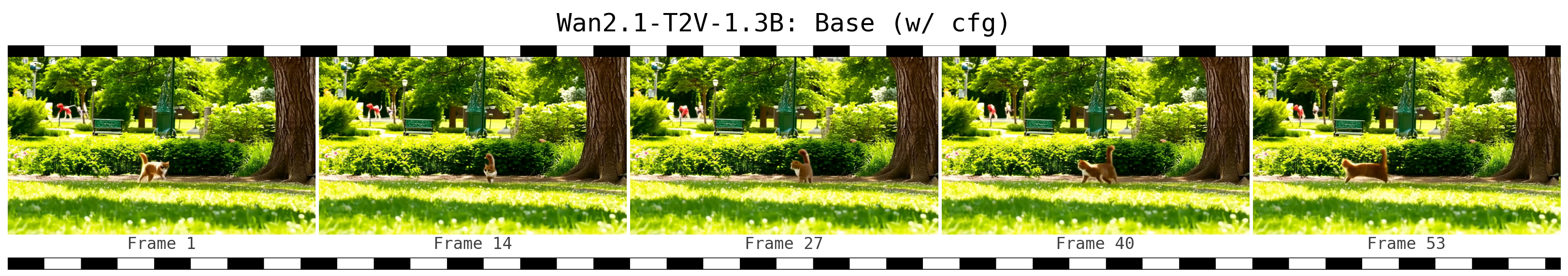}\\[1pt]
    \includegraphics[width=0.92\linewidth]{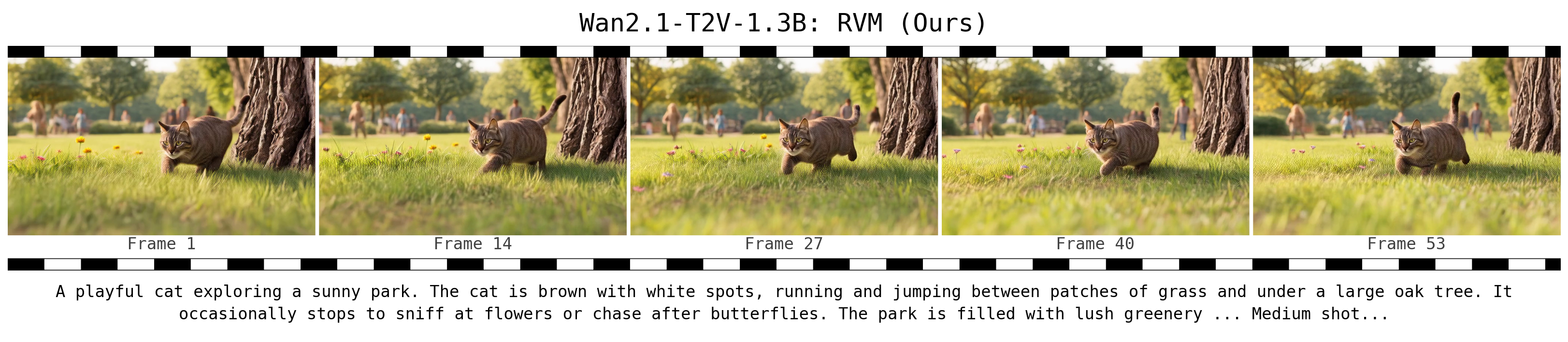}
    \caption{\textbf{Additional qualitative comparison on Wan2.1-T2V-1.3B} Frames $1, 14, 27, 40, 53$ of each method for two prompts; within each group of three rows, top to bottom: Wan2.1-1.3B (w/o CFG), Wan2.1-1.3B (w/ CFG), Ours (RVM).}
    \label{fig:app_corgi}
\end{figure*}

\begin{figure*}[h]
    \centering
    \includegraphics[width=0.92\linewidth]{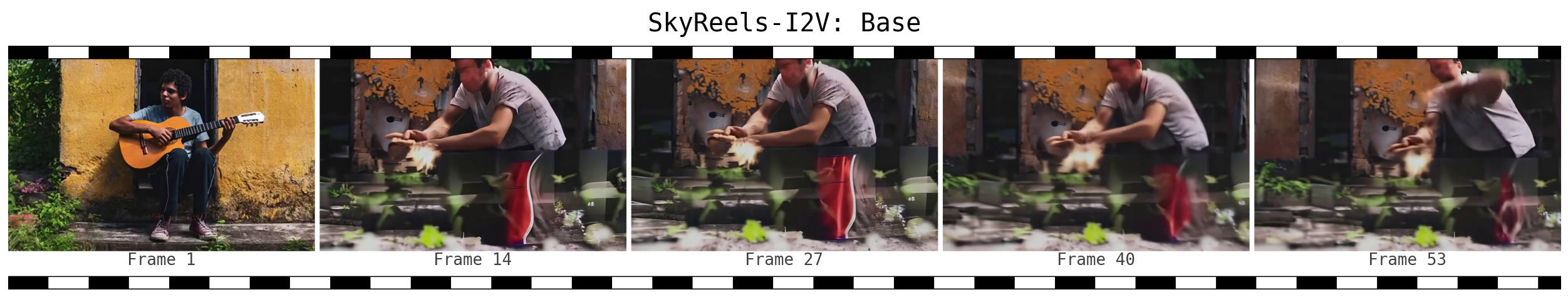}\\[1pt]
    \includegraphics[width=0.92\linewidth]{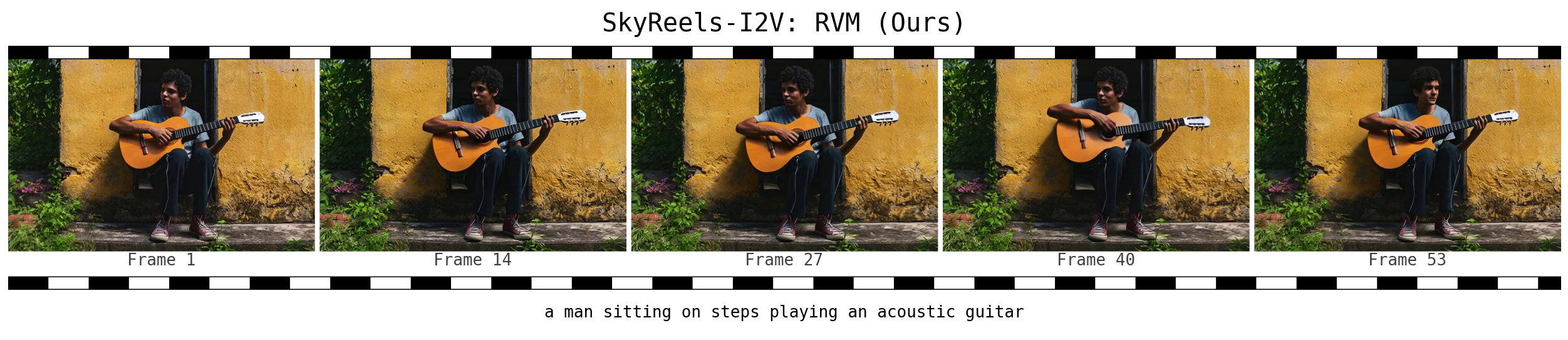}\\
    \includegraphics[width=0.92\linewidth]{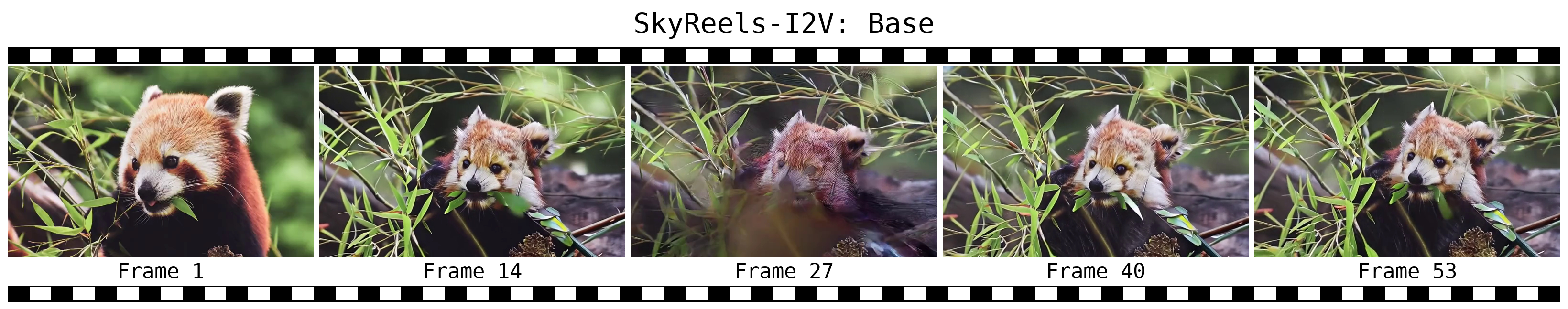}\\[1pt]
    \includegraphics[width=0.92\linewidth]{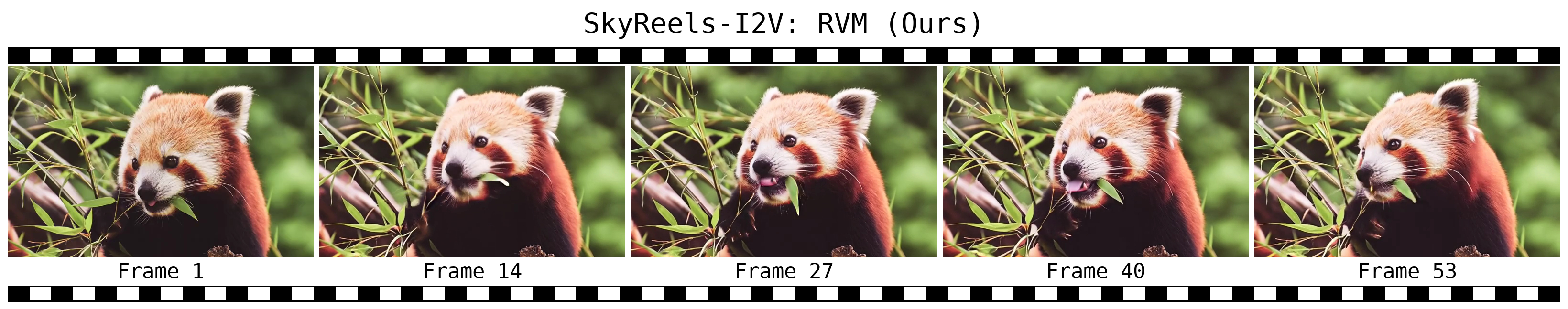}\\
    \includegraphics[width=0.92\linewidth]{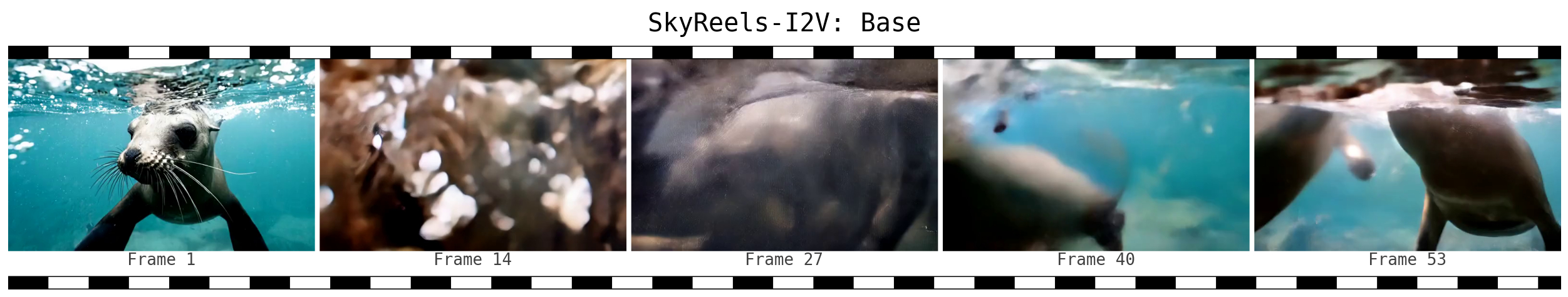}\\[1pt]
    \includegraphics[width=0.92\linewidth]{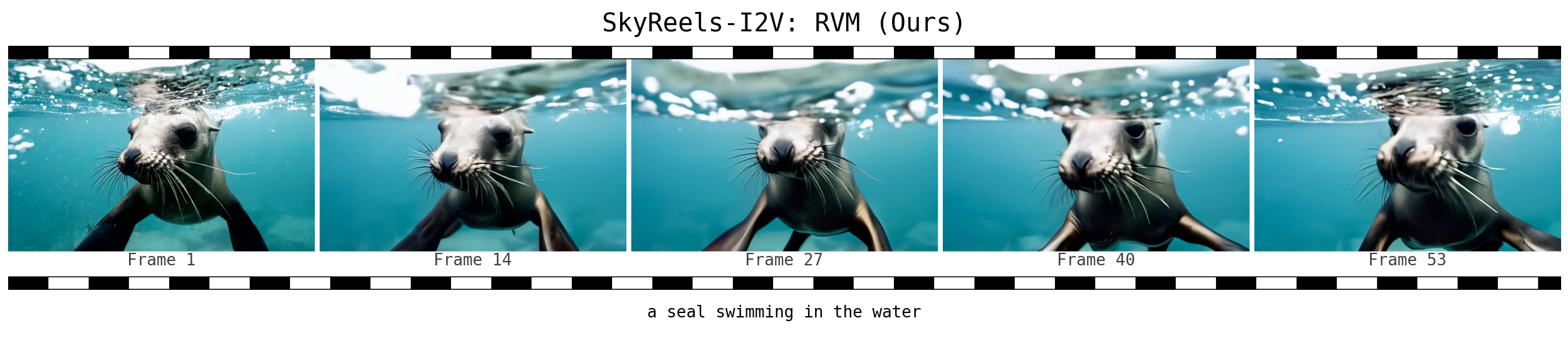}
    \caption{{\textbf{Additional qualitative comparison on SkyReels-I2V} Frames $1, 14, 27, 40, 53$ for three prompts; within each pair, SkyReels-I2V-V1 (top) and Ours (RVM) (bottom).}}
    \label{fig:app_guitar}
\end{figure*}

\begin{figure*}[h]
    \centering
    \includegraphics[width=0.92\linewidth]{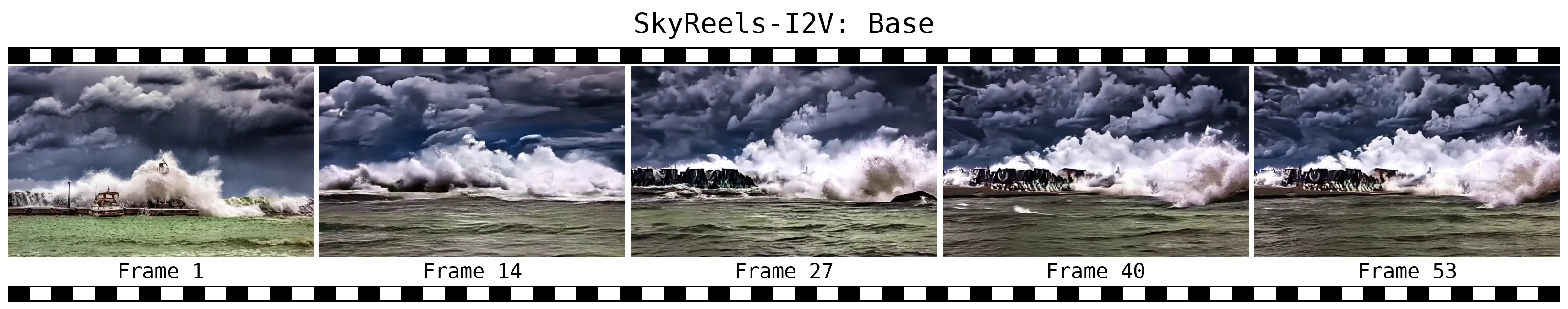}\\[1pt]
    \includegraphics[width=0.92\linewidth]{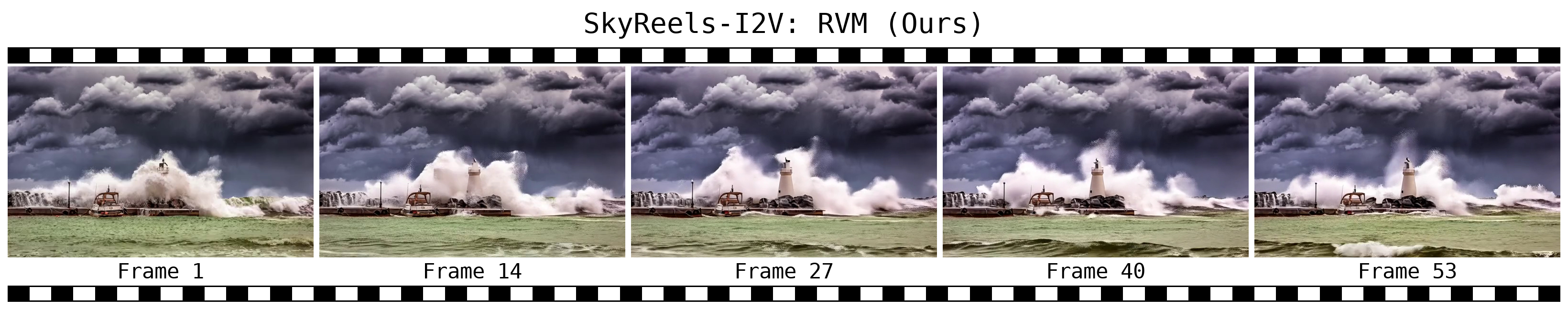}\\
    \includegraphics[width=0.92\linewidth]{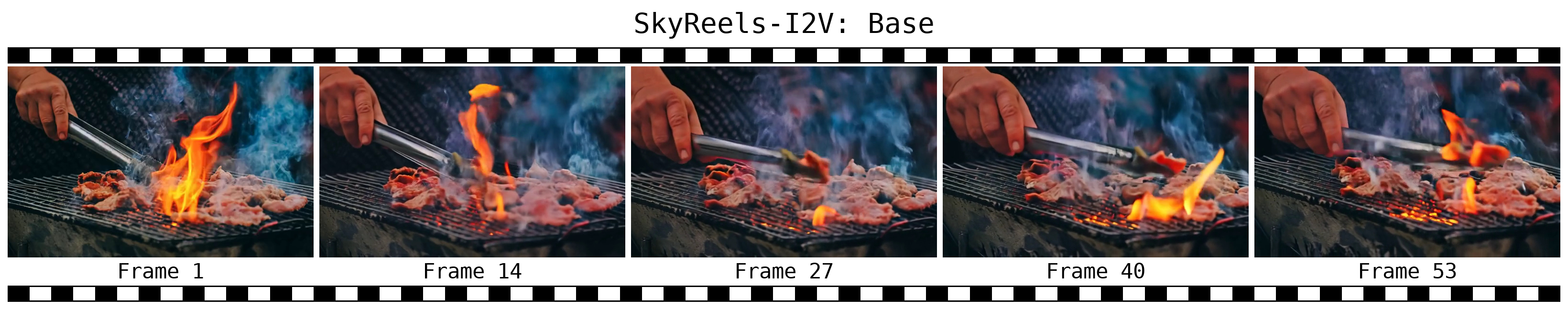}\\[1pt]
    \includegraphics[width=0.92\linewidth]{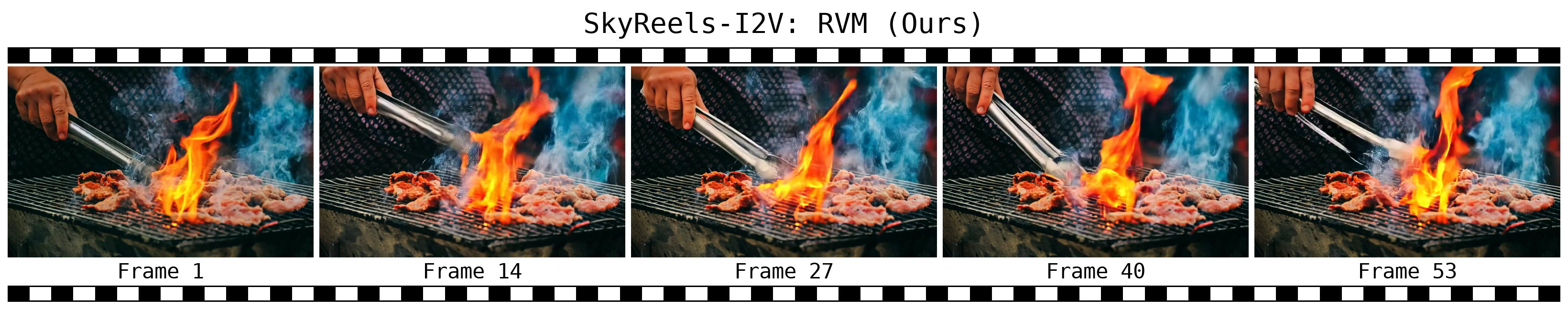}\\
    \includegraphics[width=0.92\linewidth]{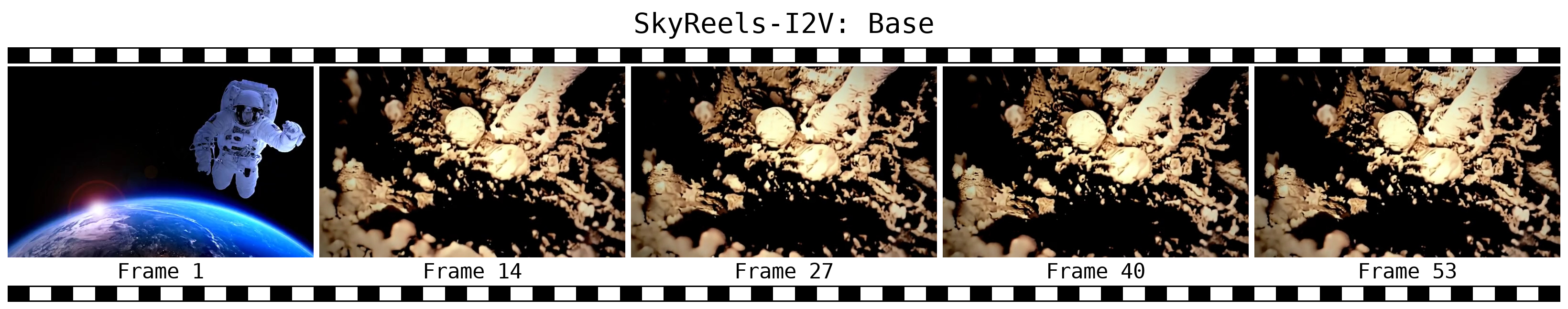}\\[1pt]
    \includegraphics[width=0.92\linewidth]{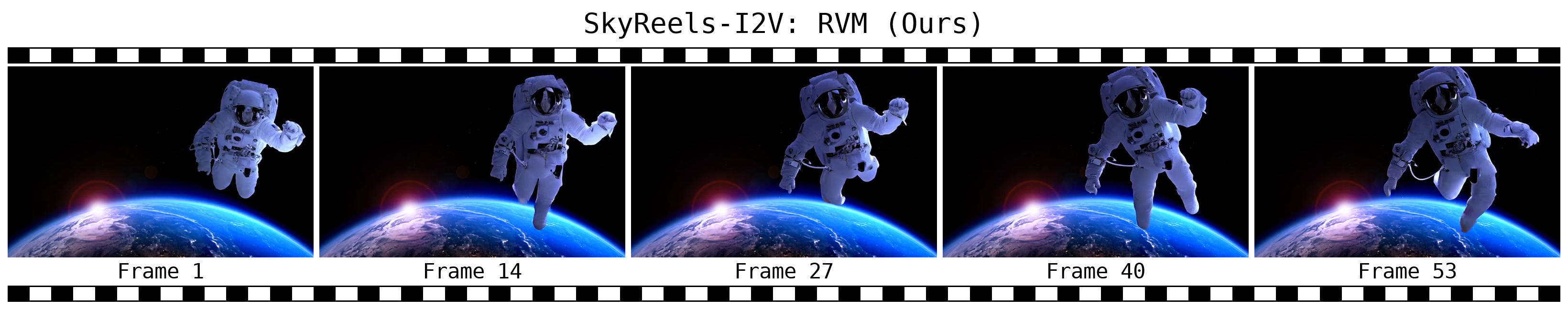}
    \caption{\textbf{Additional qualitative comparison on SkyReels-I2V} Frames $1, 14, 27, 40, 53$ for three prompts; within each pair, SkyReels-I2V-V1 (top) and Ours (RVM) (bottom).}
    \label{fig:app_sky_more}
\end{figure*}

\end{document}

%% file: subtex/package.tex
\usepackage{amsmath,amsthm,amssymb,bm}
\usepackage{amsfonts}
\usepackage{thmtools} 
\usepackage{bbm}
\usepackage{lipsum}
\usepackage{nicefrac}
\usepackage{algorithm}
\usepackage{algorithmic}
\usepackage{placeins}
\theoremstyle{plain}  % You can change the style if needed
\newtheorem{theorem}{Theorem}[section]  % Defines the theorem environment
\newtheoremstyle{remarkstyle}
  {} % Space above
  {} % Space below
  {} % Body font
  {} % Indent amount
  {\bfseries} % Theorem head font
  {.} % Punctuation after theorem head
  {.5em} % Space after theorem head
  {} % Theorem head spec (can be left empty, meaning `normal')
  
\theoremstyle{remarkstyle} 

\definecolor{mygray}{gray}{0.95}
\definecolor{mygray}{gray}{0.95}
\newcommand{\graybox}[1]{%
\begingroup
\setlength{\fboxsep}{0pt}%  
\colorbox{mygray} {% Set's the color of minipage
\begin{minipage}{\linewidth}% 	% Starts minipage
\vspace{-0.5em}%
{#1}%
\end{minipage}%
}%			% End minipage
\endgroup
}
\usepackage{color,soul}
\colorlet{mygreen}{green!55!black}
\definecolor{myyellow1}{HTML}{D89A3C}
\colorlet{myyellow}{myyellow1!90!black}
\colorlet{myblue}{blue!75!green!80!black}
\colorlet{metablue}{blue!60!green!80!black}
\definecolor{nicerblue}{HTML}{417481} % nicer blue

\usepackage{empheq}

%% file: subtex/math.tex
    {%
        \end{gathered}\end{equation}
    }

\def\1{\bm{1}}

\def\vtheta{{\bm{v}_{\theta}}}

\DeclareMathAlphabet{\mathsfit}{\encodingdefault}{\sfdefault}{m}{sl}
\SetMathAlphabet{\mathsfit}{bold}{\encodingdefault}{\sfdefault}{bx}{n}

\newcommand{\pdata}{p_{\rm{data}}}
\newcommand{\E}{\mathbb{E}}

\newcommand{\R}{\mathbb{R}}

\newcommand{\KL}{\mathrm{KL}}

\DeclareMathOperator*{\argmin}{argmin}

%% file: subtex/macro.tex
\usepackage{xspace}